\documentclass{article}

\usepackage{microtype}
\usepackage{graphicx}
\usepackage{subfigure}
\usepackage{booktabs} 

\usepackage{hyperref}

\usepackage[accepted]{icml2023}

\usepackage{microtype}
\usepackage{graphicx}
\usepackage{booktabs} 
\usepackage{multicol}
\usepackage{comment}
\usepackage{soul}
\usepackage{hyperref}
\usepackage{amsmath}
\usepackage{amssymb}
\usepackage{mathtools}
\usepackage{amsthm}
\usepackage{bbm}
\usepackage{bm}
\usepackage{xcolor}
\usepackage{mathtools}
\usepackage{mathrsfs} 
\usepackage[export]{adjustbox}
\usepackage{multirow}
\usepackage{soul}

\usepackage[capitalize,noabbrev]{cleveref}

\theoremstyle{plain}
\newtheorem{theorem}{Theorem}[section]
\newtheorem{proposition}[theorem]{Proposition}
\newtheorem{lemma}[theorem]{Lemma}
\newtheorem{corollary}[theorem]{Corollary}
\theoremstyle{definition}
\newtheorem{definition}[theorem]{Definition}
\newtheorem{assumption}[theorem]{Assumption}
\theoremstyle{remark}
\newtheorem{remark}[theorem]{Remark}
\DeclareMathOperator{\Tr}{Tr}
\DeclareMathOperator{\poly}{poly}

\newcommand{\inner}[2]{\left\langle#1,#2\right\rangle}

\newcommand{\mtx}{\bm} 
\newcommand{\vct}{\bm} 
\newcommand{\tmtx}[1]{\tilde{\mtx{#1}}}
\newcommand{\tvct}[1]{\tilde{\vct{#1}}}

\newcommand{\E}{\mathbbm{E}}

\newcommand{\aug}{\text{aug}}
\newcommand{\rank}{\textbf{rank}}

\newcommand{\badaug}[1]{\mathscr{A}'(#1)}

\newcommand{\Daug}{\hat{\mathcal{D}}_\aug}

\newcommand{\orig}{\text{orig}}

\newcommand{\sub}{\text{sub}}
\newcommand{\unsup}{\text{UCL}}
\newcommand{\supcon}{\text{SCL}}
\newcommand{\intrpl}{\text{joint}}
\newcommand{\cat}[2]{
\begin{bmatrix}#1 \\ #2
\end{bmatrix}}

\newcommand{\linspan}{\textbf{span}}
\newcommand{\colsp}{\textbf{colsp}}

\usepackage{color}

\DeclareMathOperator{\im}{im}

\usepackage[textsize=tiny]{todonotes}

\icmltitlerunning{Which Features are Learned by Contrastive Learning?}

\begin{document}

\twocolumn[
\icmltitle{Which Features are Learnt by Contrastive Learning? \\On the Role of Simplicity Bias in Class Collapse and Feature Suppression}



\icmlsetsymbol{equal}{*}

\begin{icmlauthorlist}
\icmlauthor{Yihao Xue}{ucla}
\icmlauthor{Siddharth Joshi}{ucla}
\icmlauthor{Eric Gan}{ucla}
\icmlauthor{Pin-Yu Chen}{ibm}
\icmlauthor{Baharan Mirzasoleiman}{ucla}
\end{icmlauthorlist}

\icmlaffiliation{ucla}{Department of Computer Science, University of California, Los Angeles, USA}
\icmlaffiliation{ibm}{IBM Research}

\icmlcorrespondingauthor{Yihao Xue}{yihaoxue@g.ucla.edu}

\icmlkeywords{Machine Learning, ICML}

\vskip 0.3in
]



\printAffiliationsAndNotice{\icmlEqualContribution} 

\begin{abstract}
Contrastive learning (CL) has emerged as a powerful technique for representation learning, with or without label supervision. However, supervised CL is prone to collapsing representations of subclasses within a class by not capturing all their features, and unsupervised CL may suppress harder class-relevant features by focusing on learning easy class-irrelevant features; both significantly compromise representation quality. Yet, there is no theoretical understanding of \textit{class collapse} or \textit{feature suppression} at \textit{test} time. 
We provide the first unified theoretically rigorous framework to determine \textit{which} features are learnt by CL. Our analysis indicate that, perhaps surprisingly, bias of (stochastic) gradient descent towards finding simpler solutions is a key factor in collapsing subclass representations and suppressing harder class-relevant features. Moreover, we present increasing embedding dimensionality and improving the quality of data augmentations as two theoretically motivated solutions to {feature suppression}. We also provide the first theoretical explanation for why employing supervised and unsupervised CL together yields higher-quality representations, even when using commonly-used stochastic gradient methods.
\end{abstract}

\section{Introduction}


Learning high-quality representations that generalize well to a variety of downstream prediction tasks has been a long-standing goal of machine learning \cite{hinton2006fast,ranzato2006efficient}.
Contrastive learning (CL) 
has emerged as an effective approach for solving this problem, both with and without supervision \cite{cl_simclr, cl_chuang_debiased_2020, cl_grill_bootstrap_2020, khosla2020supervised}. 
Unsupervised CL learns representations of training examples by maximizing agreement between augmented views of the same example. Similarly, supervised CL maximizes agreement between augmented views of examples in the same class.
Despite their empirical success, 
both supervised and unsupervised contrastive learning 
fail to capture \textit{all} semantically relevant features in the data.
In particular, \textit{supervised} CL can fall prey to \textit{class collapse} 
\cite{cc_dissecting_scl_2021, cc_chen2022perfectlybalanced}, where representations of \textit{subclasses} within a class may no longer be distinguishable from each other; thus, yielding a poor classification performance at the subclass level.
Similarly, \textit{unsupervised} CL can be afflicted with \textit{feature suppression} \cite{chen2021intriguing,shortcut2021_robinson} where easy but class-irrelevant features suppress the learning of harder class-relevant ones; deteriorating the generalizability of the obtained representations. \looseness=-1


 In spite of the significance of these failure modes, there is no clear theoretical understanding of them and consequently, no rigorous solution. {Feature suppression} has not been studied theoretically by prior work and the only theoretical work on {class collapse} \cite{cc_dissecting_scl_2021} cannot explain why we observe {class collapse} at \textit{test time}.
 

Addressing {class collapse} and {feature suppression} requires a theoretical understanding of \textit{which} features CL learns. However, existing CL theory \cite{cl_wang2020_uniformity,cc_dissecting_scl_2021, lee2021predicting,tosh2021contrastive,tosh2021contrastive2,arora2019theoretical,tsai2020self,haochen2021provable,wen2021toward,ji2021power} only explains {how} semantically relevant features are learned. 
The implicit assumption is that \textit{all} semantically relevant features are learned, but the occurrence of {class collapse} and {feature suppression} proves otherwise. We propose the first unified (i.e. for both supervised and unsupervised CL) framework to answer \textit{which} semantically relevant features are learned. We then leverage this framework to characterize {class collapse} and {feature suppression}. Table \ref{tab:summary} summarizes the main findings in this paper, which are detailed below.


\begin{table*}[!t]
    \caption{
    A concise overview of the key findings in this research. In the table, `CC' and `FS' refers to class collapse and feature suppression, respectively. `Thm' and `Exp' refers to theorem and experiment, respectively. 
    }
    \label{tab:summary}
    \centering
    \begin{tabular}{|l|p{3in}|l|p{1.8in}|}
    \toprule
     Loss & Finding & Thm/Exp & Implication \\
       \hline
       \multirow{4}{*}{SCL} & $\min$ loss $\not \Rightarrow$ CC &  Thm \ref{thm: minimizer_without_cc}  & \multirow{4}{1.4in}{Simplicity bias of (S)GD contributes to CC} \\
       \cline{2-3}
       & ($\min$ loss $\&~\min$ norm) $\Rightarrow$CC & Thm  \ref{thm: minimizer_with_cc} \& \ref{thm: min_norm_asymp} &  \\
       \cline{2-3}
        & (S)GD learns subclasses early in training & Thm \ref{thm: early_in_training} \& Exp & \\
        \cline{2-3}
         & (S)GD eventually unlearns subclasses, leading to CC & Exp &  \\
        \hline
         \multirow{2}{*}{UCL} &  
        With insufficient embedding size, ($\min$ loss $\&~\min$ norm) $\Rightarrow$ FS & Thm \ref{thm: FS_1} \& Exp & \multirow{2}{1.8in}{Simplicity bias of (S)GD contributes to FS; Larger embedding size/better augmentation alleviates FS}  \\
        \cline{2-3}
         & 
        With imperfect data augmentation, 
        ($\min$ loss $\&~\min$ norm) $\Rightarrow$ FS, even with sufficient embedding size &  Thm \ref{thm: FS_2} & \\
        \hline
         Joint & Joint loss can avoid both CC and FS & Thm \ref{thm: joint} \& Exp & Justification of joint loss \\ \bottomrule
    \end{tabular}
    \vspace{-4mm}
\end{table*}





\textbf{Class Collapse in Supervised CL.} We prove that, perhaps surprisingly and in contrast to the current understanding \cite{cc_dissecting_scl_2021}, global minimizers of the supervised contrastive loss do not necessarily collapse the representations of the subclasses 
at \textit{test} time. We find, however, that the \textit{minimum norm} global minimizer does suffer from class collapse on test data. 

We then study minimizing the supervised contrastive loss using (S)GD and show that, interestingly, subclass features are learned early in training. However, we verify empirically, that as training proceeds, (S)GD forgets the learned subclass features and collapses class representations.

Altogether, our findings indicate that the bias of SGD towards finding simpler solutions \cite{lyu2021gradient} is the main deriving factor in collapsing class representations.


\textbf{Feature Suppression in Unsupervised CL.} 
We provide the first theoretical characterization of feature suppression in unsupervised CL. In particular, we show that the \textit{minimum norm} global minimizer of the unsupervised contrastive loss results in feature suppression, when the embedding dimensionality is small \textit{or} when data augmentations preserve class-irrelevant features better than class-relevant features. 
Again, our results identify the simplicity bias of (S)GD as a key factor in suppressing features of the input data. In addition, our findings suggest practical solutions to the problem of feature suppression: increasing embedding dimensionality and/or improving the quality of data augmentations. 

\textbf{Theoretical Justification for Combining Supervised and Unsupervised CL to Obtain Superior Representations.} 
Finally, we prove that the \textit{minimum norm} global minimizer of the joint loss (weighted sum of the supervised and unsupervised contrastive loss) does not suffer from class collapse or feature suppression, explaining why \citet{cc_chen2022perfectlybalanced, islam2021broad} observes this empirically (i.e. even when using SGD).

\section{Related Work}
\textbf{Theory of CL.}
While there has been much progress in theoretically understanding CL, most prior work \cite{cl_wang2020_uniformity,cc_dissecting_scl_2021, lee2021predicting,tosh2021contrastive,tosh2021contrastive2,arora2019theoretical,tsai2020self,haochen2021provable} are focused on understanding how CL clusters examples using semantically meaningful information or providing generalization guarantees on downstream tasks. Feature learning has only been studied by
\cite{wen2021toward,ji2021power} which show that CL learns semantically meaningful features from the data. In contrast, we show that CL may not learn \textit{all} semantically relevant features. 

Other important recent work \cite{saunshi2022understanding,haochen2022theoretical} studied the role of inductive bias of the function class in the success of CL. Our analysis, however, is focused on understanding failure modes of CL i.e. {class collapse} and {feature suppression}.

\textbf{Class Collapse in Supervised CL.} \citet{cc_chen2022perfectlybalanced} empirically demonstrates \textit{class collapse} on test data, but does not offer any rigorous theoretical explanation. \citet{cc_dissecting_scl_2021} proves that optimizing the supervised contrastive loss leads to class-collapsed training set representations. 
However, we show that there exist many minimizers with such class-collapsed training set representations and not all of them suffer from class collapse at \textit{test time}. We also present the first theoretical characterization of {class collapse} at test time. \looseness=-1

\textbf{Feature Suppression in Unsupervised CL.} Feature suppression has been empirically observed by \citet{tian2020makes,chen2021intriguing,shortcut2021_robinson} but we lack a theoretical formulation of this phenomenon. \citet{addressing_feature_suppression_2020} shows that InfoNCE has local minimums that exhibit {feature suppression}, thus attributing this phenomenon to failure of optimizing the loss. However, \citet{shortcut2021_robinson} shows that the InfoNCE loss can be minimized by many models, some of which learn all task-relevant features, while others do not. We put forth the only theoretical characterization of {feature suppression} and consequently, use this understanding to suggest practical solutions to remedy this problem. \looseness=-1

\textbf{Joint Supervised and Unsupervised Contrastive Loss.} Recently, several versions of loss functions that combine supervised and unsupervised contrastive losses \cite{islam2021broad, cc_chen2022perfectlybalanced} have been empirically observed to have superior transfer learning performance, by avoiding class collapse. We provide the first theoretically rigorous analysis of which features the \textit{minimum norm} global minimizer of the joint loss learns, provably demonstrating that it can avoid class collapse and feature suppression. To the best of our knowledge, this is the only theoretical result that can be used to understand the empirical success of joint losses. \looseness=-1

\section{Problem Formulation}

\subsection{Data distribution}\label{sec: data_dist}

We define data distribution $ \mathcal{D}_{\orig}$ below. Each example $(\vct{x}, y, y_\sub)\in\mathcal{D}_\orig$ is generated as follows:
\begin{align}
    \nonumber
    \vct{x} = & \vct{u} + \vct{\xi}, \quad\quad \text{where}~~\\
    \nonumber
     \vct{u} = (y \phi_1+\mu_1)\vct{v}_1 +& (y_\sub \phi_2+ \mu_2)\vct{v}_2
    + (\rho_k\phi_{k}+\mu_k) \vct{v}_{k},
\end{align}
and $k$ is uniformly selected from $3, \dots, K$; and $y, y_\sub, \rho_k$ are uniformly sampled from $\{-1, 1\}$. 

\textbf{Features and Noise.} We assume features and noise form an orthonormal basis of $\mathbbm{R}^d$, i.e., a set of unit orthogonal vectors $\{\vct{v}_1, \dots, \vct{v}_d\}$ in $\mathbbm{R}^d$. W.l.o.g., one can let $\vct{v}$'s be the standard basis, where the first $K$ 
basis are feature vectors. $\{\phi_1, \dots, \phi_K\}$ are constants indicating the strength of each feature, and $\{\mu_1, \dots,\mu_K\}$ are the means of the corresponding entries in the feature vectors. In particular: 
    

    $\bullet$ Class Feature: $\vct{v}_1$.\\
$\bullet$  Subclass Feature: $\vct{v}_2$. \\
$\bullet$ (Class and subclass) irrelevant features:\footnote{In the rest of the paper, we use irrelevant features to refer to features that may have semantic meaning but are irrelevant to class and subclass.}  $\vct{v}_3, \dots, \vct{v}_K$.\\
$\bullet$ Noise $\vct{\xi}\sim\mathcal{D}_\xi$:  $\mathcal{D}_\xi$ is a uniform distribution over features $ \sigma_\xi\vct{v}_1, \dots, \sigma_\xi\vct{v}_d $, where $\sigma_\xi$ indicates the variance of the noise.\footnote{This definition of noise is nearly identical to Gaussian noise $\mathcal{N}(0,\frac{\sigma_\xi^2}{d}\mathbf{I}_d)$ in the high-dimensional regime but keeps the analysis clear. Our results can be extended to the Gaussian noise setting. } \looseness=-1

We sample $n$ examples from $\mathcal{D}_\orig$ to form the original dataset $\hat{\mathcal{D}}_\orig$. 
\begin{assumption}[Balanced Dataset] \label{assump: balanced}
All combinations of  $(y_i, y_{\sub, i},  k_i, \rho_i)$ are equally represented in $\hat{\mathcal{D}}_\orig$.\footnote{This can be approximately achieved when $n$ is sufficiently larger than $K$. While our analysis can be generalized to consider imbalanced data, this is outside the scope of this work. 
}
\end{assumption}

\textbf{A Concrete Example of the Above Data Distribution.} Let $y=1$ be dogs and $y=-1$ be cats, $y_\sub=1$ if they are fluffy and $y_\sub=-1$ if they are not-fluffy. Then $(\phi_1+\mu_1)\vct{v}_1+(\phi_2+\mu_2)\vct{v}_2$ denotes a fluffy dog. Here, the background can be interpreted as an irrelevant feature: let $\rho_3=1$ for grass and $\rho_3=-1$ for forest. Then $(\phi_1+\mu_1)\vct{v}_1+(\phi_2+\mu_2)\vct{v}_2 + (\phi_3 + \mu_3) v_3$ represents a fluffy dog on grass. Note that each example only selects one irrelevant feature, which mimics the real world, where examples do not necessarily have all types of objects in the background i.e. many examples have neither grass or forests as their background. 

\textbf{Rationale for Including Feature Means $\mu_i$.}
In general, it is unreasonable to expect all features to have $0$ expectation over entire data, thus we introduce $\mu$ to further generalize our analysis. 
We find that considering a non-zero mean for the subclass feature is sufficient to provide novel insights into class collapse (Theorem \ref{thm: early_in_training}). Therefore, for clarity, we set all the $\mu$'s except $\mu_2$ to zero.

\textbf{Relation to Sparse Coding Model.} This data distribution is a variant of the sparse coding model which is usually considered as a provision model for studying the feature learning process in machine learning (e.g., \citep{zou2021understanding,wen2021toward, liu2021self}). It naturally fits into many settings in machine learning, and in general mimics the outputs of intermediate layers of neural networks which have been shown to be sparse \cite{papyan2017convolutional}. It is also  used to model the
sparse occurrences of objects in image tasks \citep{olshausen1997sparse,vinje2000sparse,foldiak2003sparse,protter2008image,yang2009linear,mairal2014sparse} and 
polysemy of words in language tasks \cite{arora2018linear}.


\subsection{Data Augmentation $\mathscr{A}(\cdot)$}
For each example in $\hat{\mathcal{D}}_\orig$, we generate $m$ augmentations 
to form $\hat{\mathcal{D}}_\aug$. We consider the following augmentation strategy: given an example $\vct{x}=\vct{u} +\vct{\xi}$, its augmentation is given by $
    \mathscr{A}(\vct{x}) = \vct{u} +\vct{\xi}' $, 
where $\vct{\xi}'$ is a new random variable from $\mathcal{D}_\xi$ 
independent of $\vct{\xi}$. This is an abstract of augmentations used in practice where two augmentations from the same example share certain parts of the features and have the correlation between their noise parts removed or weakened.

\begin{assumption}[High dimensional regime] 

$d$ is at least $\omega(n^2m^2)$.

\end{assumption}

\begin{assumption}[Sufficient sample size]\label{assump: large_sample_size}
The noise-to-sample-size ratio is not too large $\frac{\sigma_\xi^2}{mn}=o(1)$.
\end{assumption}


\subsection{Linear Model}
We consider a linear model with $p$ outputs. The model has weights $\vct{W}\in\mathbbm{R}^{p\times d}$ and bias $\vct{b}\in\mathbbm{R}^p$ where $p \geq 3$. The function represented by the model is $
    f_{\vct{\Theta}}( \vct{x}) = \mtx{W}\vct{x} + \vct{b}$, where we define $\vct{\Theta} \in\mathbbm{R}^{p\times(d+1)} $ as the concatenated parameter $ [\vct{W}~~ \vct{b}]$. We establish theoretical proofs of class collapse and feature suppression for linear model, and also empirically verified them for (non-linear) deep neural networks.
    \looseness=-1

\subsection{Loss function}\label{sec: loss}
 For unsupervised contrastive learning, we use the unsupervised spectral contrastive loss popular in prior theoretical and empirical work \cite{haochen2021provable,saunshi2022understanding,haochen2022theoretical} and for supervised contrastive learning, we consider the natural generalization of this loss to incorporate supervision. 
Let $\mathcal{A}_i$ denote the set of augmentations in $\Daug$ generated from the $i$-th original example with $\mathscr{A}(\cdot)$. Let $\mathcal{S}_{+1}$ and $\mathcal{S}_{-1}$ denote the set of augmentations in $\Daug$ with class labels $+1$ and $-1$, respectively. Let $\hat{\mathbbm{E}}$ denote the empirical expectation. Then we have the following loss functions: 
\begin{align}
    \nonumber
    \mathcal{L}_\unsup (\vct{\Theta}) = &  -2\hat{\mathbbm{E}}_{i\in[n],\vct{x}\in\mathcal{A}_i, \vct{x}^+\in\mathcal{A}_i}\left[f_{\Theta}(\vct{x})^{\top} f_{\Theta}(\vct{x}^+)\right] \\
     +& \hat{\mathbbm{E}}_{\vct{x}\in\Daug, \vct{x}^-\in\Daug}\left[(f_{\Theta}(\vct{x})^{\top} f_{\Theta}(\vct{x}^-))^2\right]  
    \label{eq: loss:ucl} \\
    \nonumber
    \mathcal{L}_\supcon (\vct{\Theta}) = &  -2\hat{\mathbbm{E}}_{c\in\{-1,1\}, \vct{x}\in\mathcal{S}_c, \vct{x}^+\in\mathcal{S}_c}\left[f_{\Theta}(\vct{x})^{\top} f_{\Theta}(\vct{x}^+)\right] \\
     +& \hat{\mathbbm{E}}_{\vct{x}\in\Daug, \vct{x}^-\in\Daug} \left[(f_{\Theta}(\vct{x})^{\top} f_{\Theta}(\vct{x}^-))^2\right].  
    \label{eq: loss:scl} 
\end{align}

\section{Simplicity Bias Contributes to Class Collapse in Supervised CL}\label{sec: cc}



We make two key observations through our theoretical analysis and experiments (henceforth we refer to class collapse at \textit{test time} simply as `class collapse'):
\looseness=-1
\begin{itemize}
    \item[1.] Theoretically, not all global minimizers exhibit class collapse, but the \emph{minimum norm} minimizer does.
    \item[2.] Theoretically and empirically, when the model is trained using (S)GD, some subclasses are \emph{provably} learned early in training. Empirically, however, those subclasses will eventually be unlearned i.e. S(GD) converges to minimizers that exhibit class collapse.\looseness=-1
\end{itemize}
Altogether, these observations suggest that class collapse, which has been observed in practice when certain gradient-based algorithms are used to minimize the loss, cannot be explained by simply analyzing the loss function. This highlights the importance of studying the dynamics and inductive bias of training algorithms in contrastive learning.\looseness=-1

\begin{figure*}[!t]
\centering
\subfigure[epoch 0]{
\includegraphics[width=0.21\textwidth]{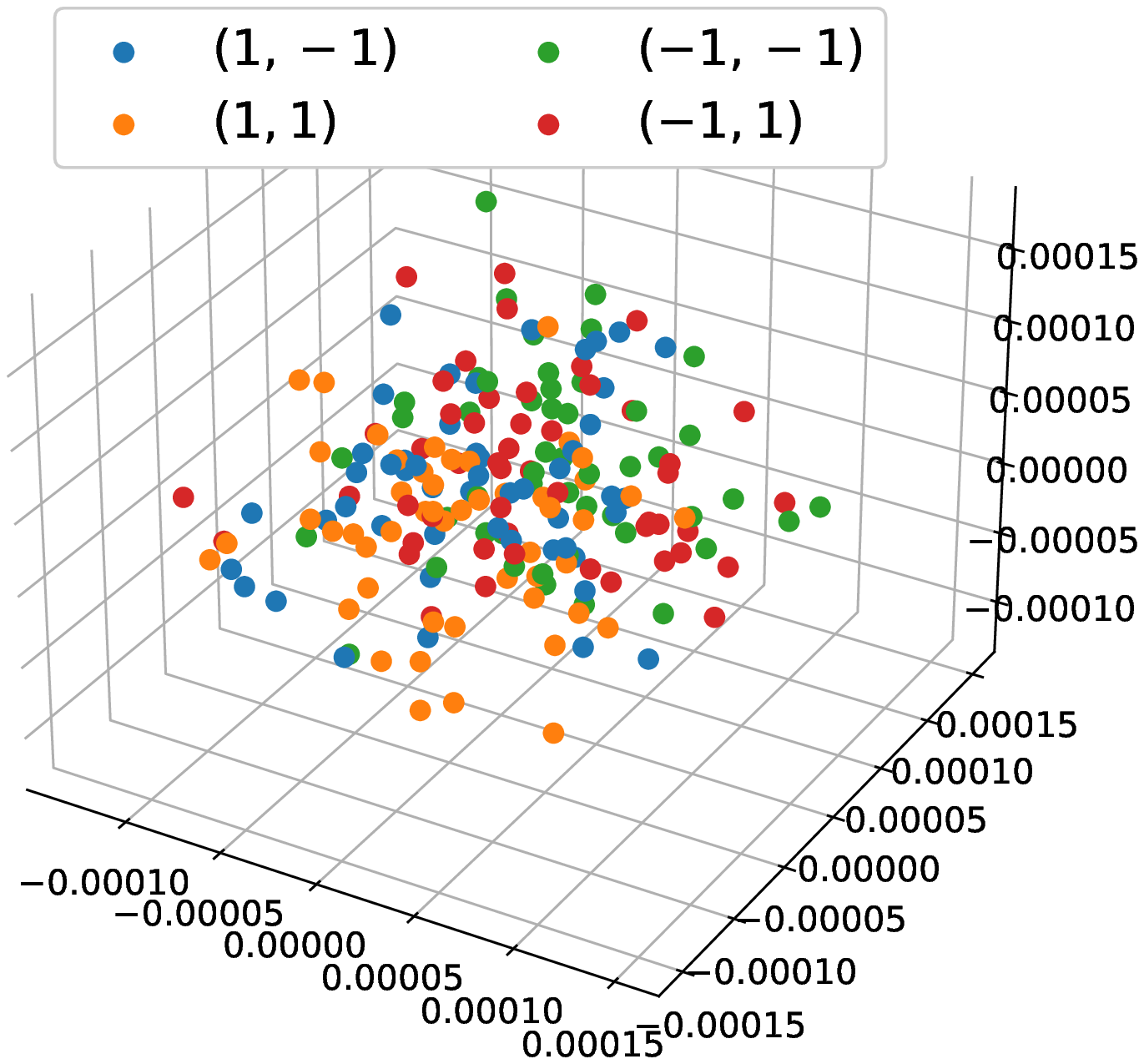}
	}
\subfigure[epoch 45]{
\includegraphics[width=0.21\textwidth]{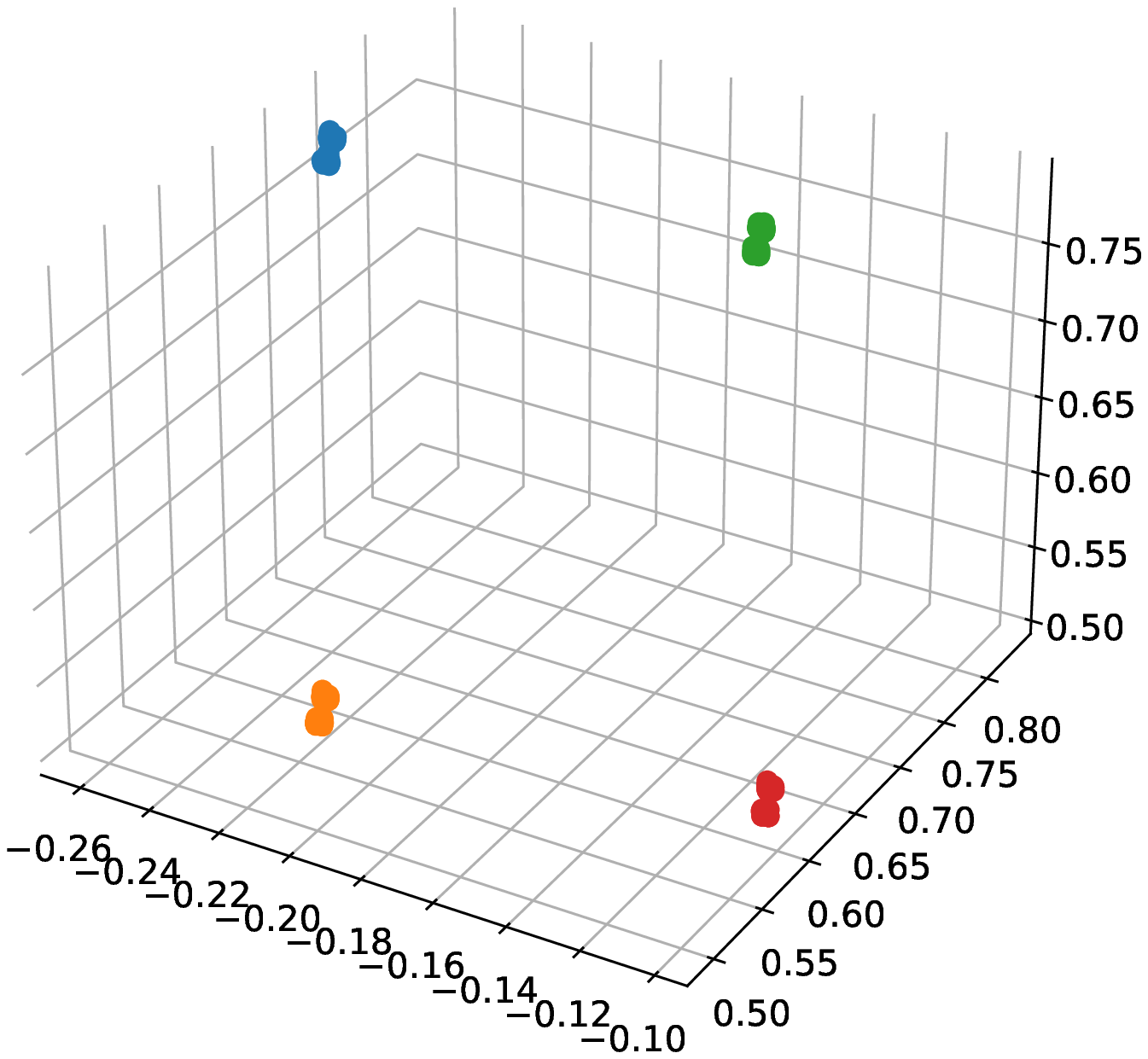}
	}
 \subfigure[epoch 60]{
\includegraphics[width=0.21\textwidth]{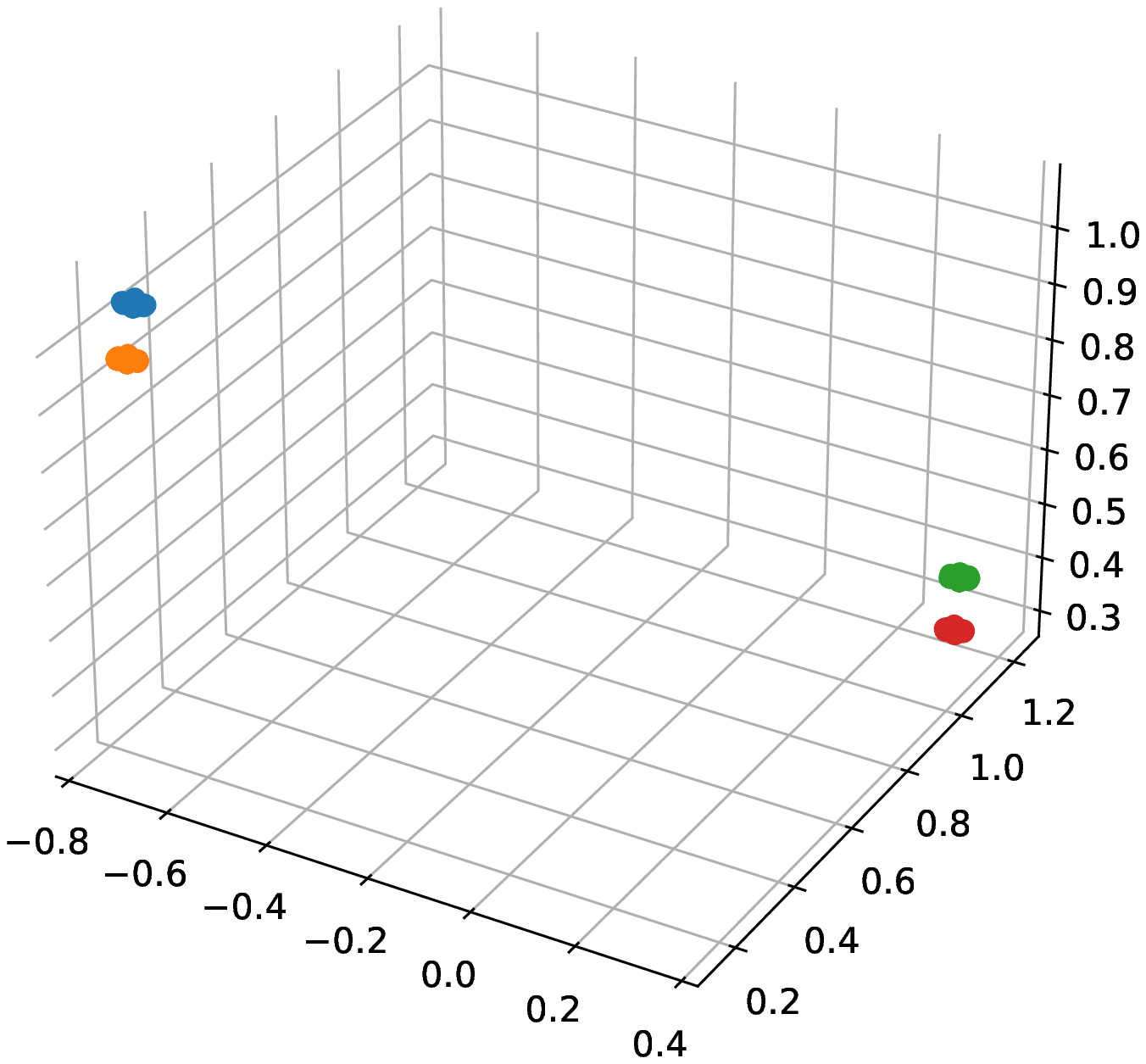}
	}
 \subfigure[epoch 100]{
\includegraphics[width=0.21\textwidth]{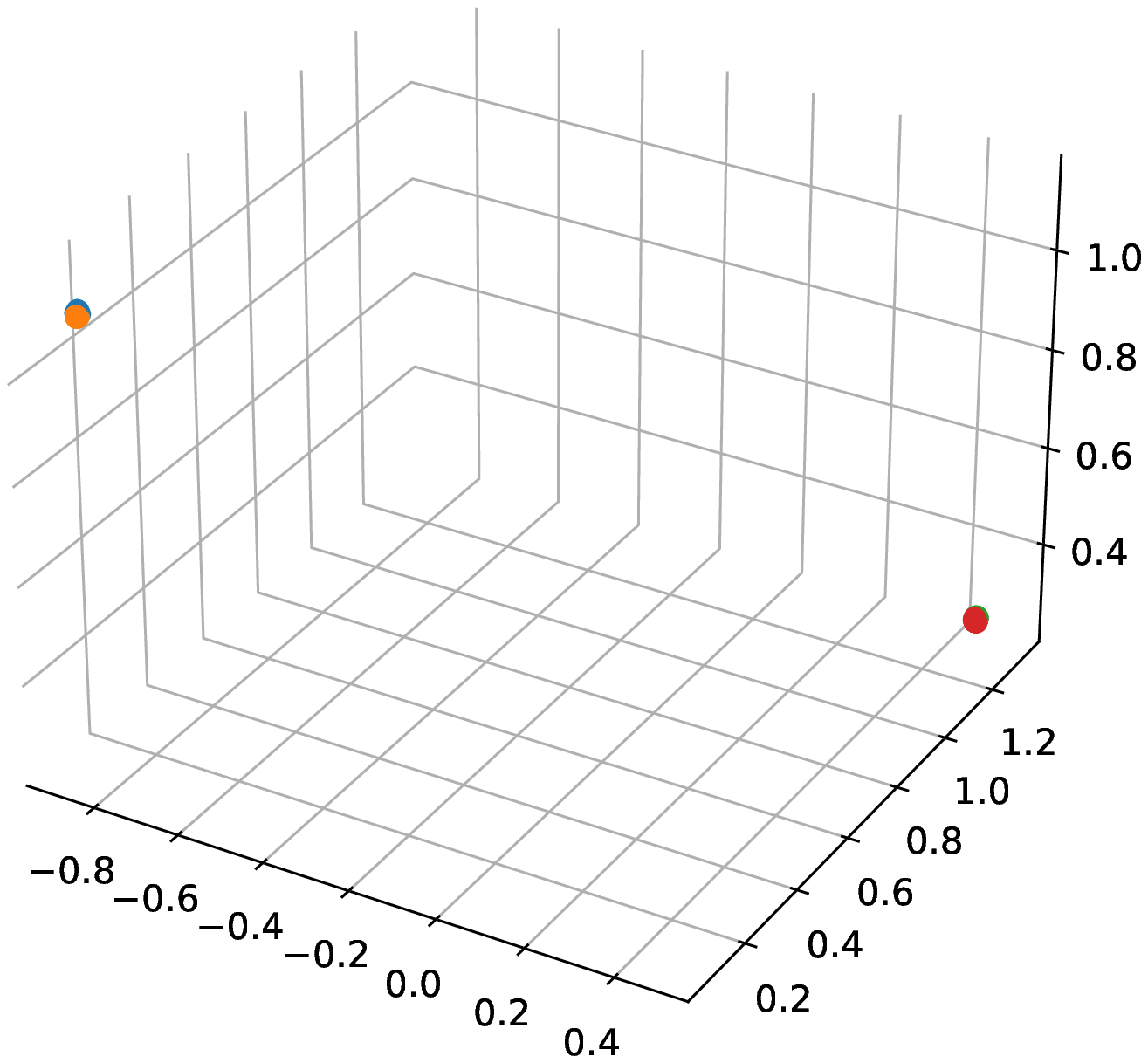}
	}
    \caption{Visualization of the embedding space at different epochs. We let $p=3$ so that we can see the whole embedding space from a 3D plot. Other parameters: $n=1000, m=5, d=2000, K=4, \phi_1=\phi_2=\phi_3=\phi_4=1, \mu=1, \sigma=2, \sigma_0=0.001, \eta=0.05$. Colors represent combinations of class and subclass labels $(y, y_\sub)$. We use test examples for the plots. At epoch 45, the four groups of examples are well separated in the embedding space. However groups in the same classes are merged afterwards. }
    \label{fig: cc_embedding}
\end{figure*}

\begin{figure}[!t]
    \centering
 \subfigure[$p=3$]{
\includegraphics[width=0.21\textwidth]{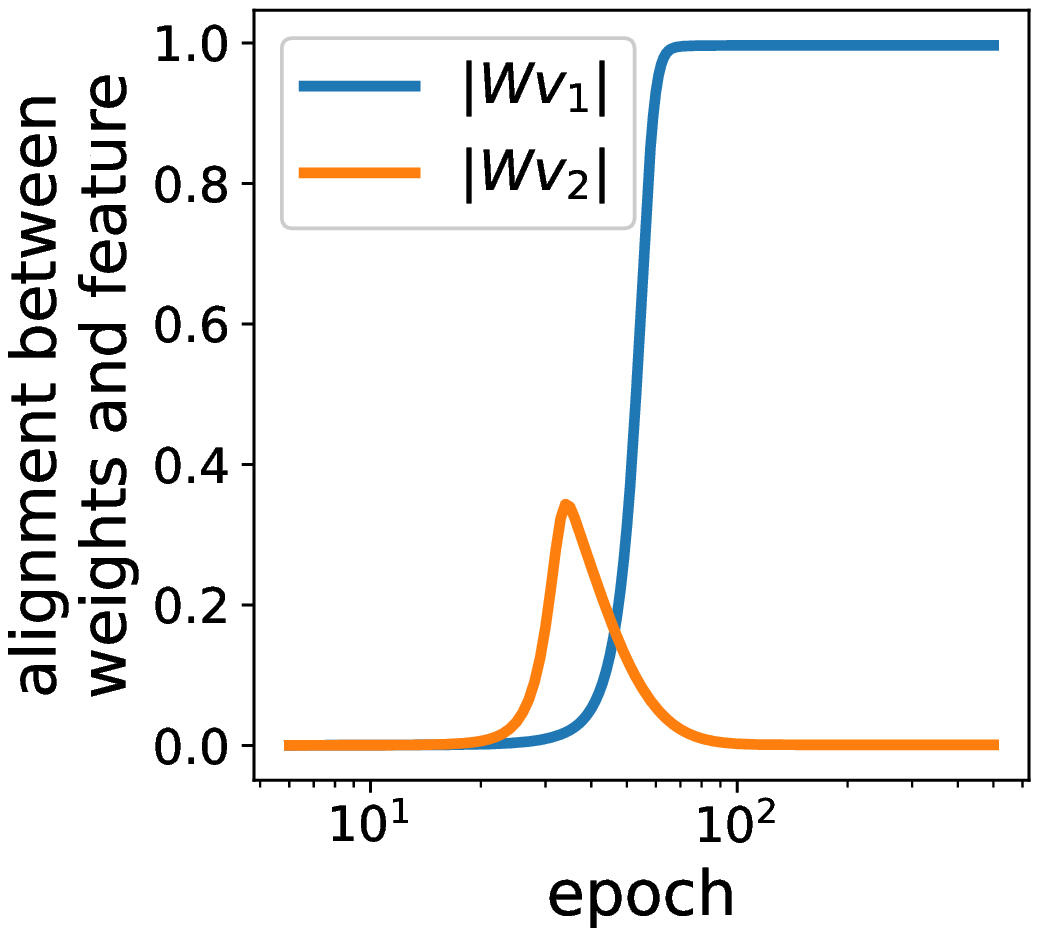}
}
 \subfigure[$p=500$]{
\includegraphics[width=0.21\textwidth]{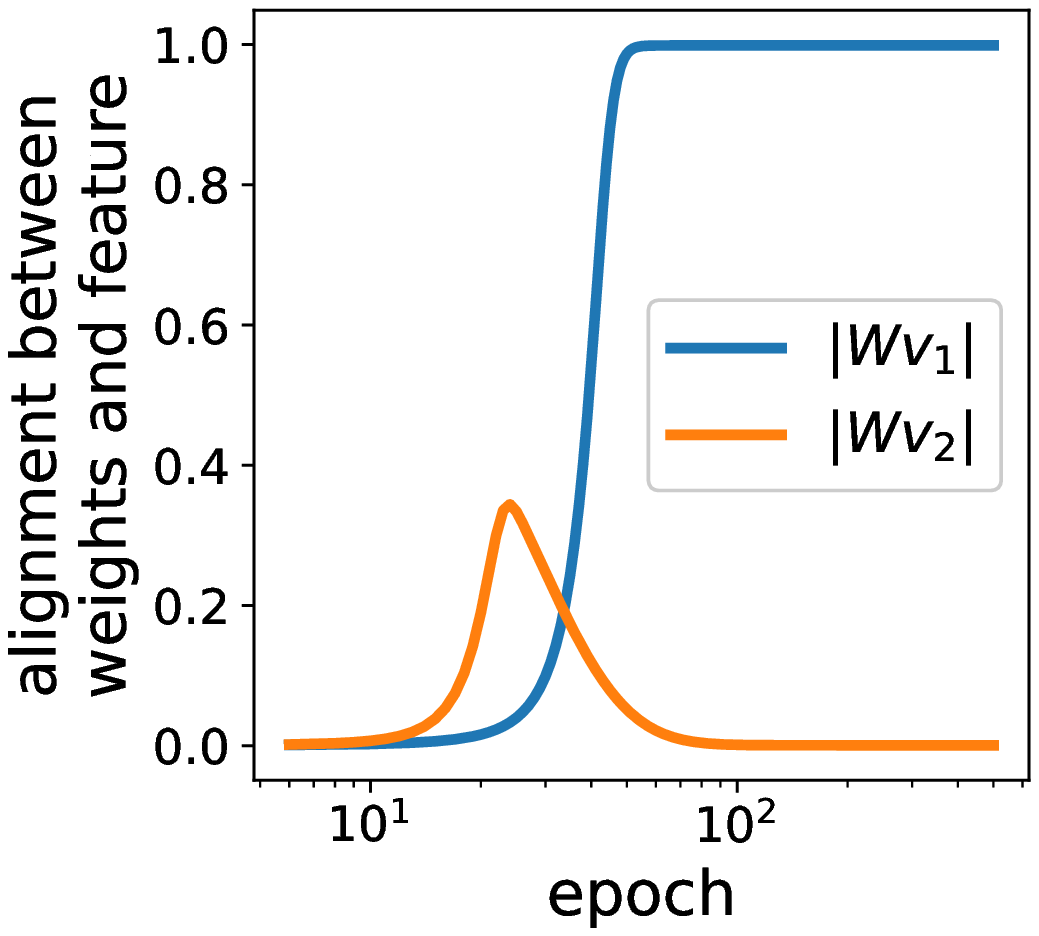}
}
\vspace{-.3cm}
    \caption{$\|\mtx{W}_t \vct{v}_1 \|$ and $\|\mtx{W}_t \vct{v}_2 \|$ at different epochs. Both features are learned early in training, but $\vct{v}_2$ is unlearned later.}
    \label{fig: cc_alignment}
    \vspace{-3mm}
\end{figure}


\subsection{What Minimizers Have Class Collapse?}
We first define class collapse in terms of the alignment between the model weights and the subclass feature.
\begin{definition}[Exact class collapse]\label{def: exact_cc} We say exact class collapse happens at test time when:
\begin{align}
\nonumber
\forall \vct{\beta}\in\mathbbm{R}^p, \Pr_{(\vct{x},y,y_\sub)\sim\mathcal{D}_{\orig}}(y_\sub\vct{\beta}^\top f_{\mtx{\Theta}}(\vct{x})>0) = 1/2.
\end{align}
\end{definition}

The definition means that no linear classifier on the embeddings of examples drawn from $\mathcal{D}_{\orig}$ can predict the subclass label with accuracy beyond random guess.\footnote{Actually we are able to analyze a stronger version of class collapse: $\Pr_{(\vct{x}, y, y_\sub)\sim\mathcal{D}_\orig}( f_{\mtx{\Theta}}(\vct{x}) | y_\sub ) =\Pr_{(\vct{x}, y, y_\sub)\sim\mathcal{D}_\orig}( f_{\mtx{\Theta}}(\vct{x})  ) $, which means the distributions of embeddings given and not given the subclass label are exactly the same. Nonetheless, we present this simpler formulation for clarity. }  

This is different from class collapse on the training set which is not defined on the population set $\mathcal{D}_{\orig}$ but on the training samples $\mathcal{\hat{D}}_{\orig}$. 

\begin{proposition}
For any $\mtx{\Theta}^*\in \min_{\mtx{\Theta}}\mathcal{L}_\supcon(\mtx{\Theta})$, we have$f_{\mtx{\Theta}^*}(\vct{x}_i) = f_{\mtx{\Theta}^*}(\vct{x}_j)$ for all $ \vct{x}_i, \vct{x}_j$ in the training set $\hat{\mathcal{D}}_\aug$ such that $ y_i = y_j$.\looseness=-1
\label{cc:train_emb}
\end{proposition}

This directly implies that minimizing the loss results in class collapse on the training set. However, the following theorem \ref{thm: minimizer_without_cc} shows that minimizing the loss does not necessarily lead to class collapse on the test set. To determine whether class collapse occurs, we need to determine whether the model learns the subclass feature. With a linear model, this exactly corresponds to constant alignment between weights and the subclass feature. 




\begin{theorem}[Minimizing $\mathcal{L}_\supcon \not \Rightarrow $ Class Collapse]\label{thm: minimizer_without_cc}

With high probability i.e. at least $1-O(\frac{m^2n^2}{d})=1-o(1)$, there exists $\mtx{\Theta}^*=[\mtx{W}^*~~ \vct{b}^*]$ such that $\mtx{\Theta}^*\in \min_{\mtx{\Theta}}\mathcal{L}_\supcon(\mtx{\Theta})$ $\vct{W}^*$ has constant alignment with subclass feature $\vct{v}_2$ i.e. $$\|\vct{W}^*\vct{v}_2\|=\Omega(1).$$
Hence, there exists a linear classifier in the embedding space that can predict subclass labels almost perfectly. I.e.,
\begin{align}
    \nonumber
    \exists \vct{\beta}, ~s.t.~  \Pr_{(\vct{x},y, y_\sub)\sim\mathcal{D}_\orig}(y_\sub\vct{\beta}^\top \mtx{W^*}\vct{x}>0 |y ) = 1-o(1).
\end{align}
\end{theorem}
We prove the theorem in Appendix \ref{apdx: cc}. The proof utilizes Lemma \ref{lemma: SVD_perturb} which implies that, due to the high-dimensionality, the noise vectors have non-trivial effects on the empirical covariance matrix by rotating its kernel space. This results in 
the kernel space to have a $\Theta(\frac{\sigma_\xi}{\sqrt{mn}})$ alignment with the subclass feature. Since minimizers of the loss can behave arbitrarily on this kernel space, without any additional restriction, they can have any alignment with the subclass feature.

Next, we show that, the \textit{minimum norm} minimizer exhibits class collapse. 

\begin{theorem}[Minimizing $\mathcal{L}_\supcon$ + Minimum Norm $\implies$ Class Collapse]\label{thm: minimizer_with_cc}

Assume $\mu_2=0$. Let $\mtx{\Theta}^{**}=[\mtx{W}^{**}~~ \vct{b}^{**}]$ be the minimum norm minimizer of $\mathcal{L}_\supcon$, i.e.,
$$
\mtx{\Theta}^{**} = \arg\min_{\mtx{\Theta^*}} \| \mtx{\Theta}^* \|_F ~s.t.~ \mtx{\Theta}^*\in\arg\min_{\mtx{\Theta}} \mathcal{L}_\supcon(\mtx{\Theta}).
$$
Then with high probability i.e. at least $1-O(\frac{m^2n^2}{d})=1-o(1)$, $\mtx{W}^{**}$ has no alignment with subclass feature $\vct{v}_2$ i.e. \looseness=-1
$$ \|\mtx{W}^{**} \vct{v}_2 \|=0.$$  

This means class collapse occurs at test time (Definition \ref{def: exact_cc}), and no linear classifier does better than random guess for predicting subclass labels.
\end{theorem}
Theorems \ref{thm: minimizer_without_cc} and \ref{thm: minimizer_with_cc} show that minimizing the training loss does not necessarily lead to class collapse on test data, but does with additional constraint on the weights of the model. This is not due to a degenerate solution, as we show that both minimizers learn the class feature $\vct{v}_1$ (see corollary \ref{corollary: v1_learned_supcon}).\looseness=-1


\subsection{Intriguing Properties of GD}\label{sec: GD}
We now further our theoretical characterization of class collapse by investigating the setting where $\mathcal{L}_\supcon$ is minimized by GD. This is an important step toward understanding class collapse in practice, where similar optimization algorithms are used to minimize the loss. Our findings indicate that it is likely the simplicity bias of commonly used optimization algorithms that eventually leads to class collapse. 

 We consider GD with a constant learning rate $\eta$. The weights are initialized from a Gaussian distribution, i.e., the initial weight $\mtx{\Theta}_0$ has each of its element drawn from $\mathcal{N}(0, \frac{\sigma_0^2}{d})$. And the weights at training epoch $t$ are given by:
\begin{align}
    \nonumber
    \mtx{\Theta}_t = \mtx{\Theta}_{t-1} -\eta \nabla_{\mtx{\Theta}}\mathcal{L}_\supcon (\mtx{\Theta}_{t-1}).
\end{align}

\paragraph{Early in Training Some Subclasses are Provably Learned.}
By analyzing the training dynamics of GD, we find that subclasses are learned early in training.

\begin{theorem}[Early in training subclass features are learned] \label{thm: early_in_training} Assume $\sigma_0\sqrt{\frac{p}{d}}=o(1)$ and $\sigma_\xi=o(1)$. If the subclass feature has a constant non-zero mean such that $1+\mu^2 > \phi_1^2 $, then with probability at least $1-O(\frac{m^2n^2}{d}+\frac{1}{\poly(p)})=1-o(1)$
the following holds:\\
$\bullet~~ \|\vct{W}_0\vct{v}_2\|=o(1)$.\\
$\bullet~~ \exists t=O(\ln(\frac{1}{\sigma_0} \sqrt{\frac{d}{p}}))$, s.t. $\| \mtx{W}_t\vct{v}_2\|=\Omega(1)$, and \\
$ \bullet~~ \exists \vct{\beta}, \!~s.t.~  \Pr_{(\vct{x},y, y_\sub)\sim\mathcal{D}_\orig}(y_\sub\vct{\beta}^\top \mtx{W}_t\vct{x}\!>\!0 |y ) \!=\! 1-o(1).
    $ 
\end{theorem}
The above theorem shows that there exists an epoch where the weights have constant alignment with the subclass feature and produce distinguishable subclass embeddings (proof in Appendix \ref{apdx: early}). 

The key step of our analysis is showing that early in training, GD aligns the weights with the first eigenvector of the covariance matrix of class centers. This alignment grows exponentially faster than alignments with any other directions. When $1+\mu^2 >\phi_1^2$, the subclass feature has a constant projection onto the first eigenvector and is therefore learned by the model. 

\begin{figure}
    \centering
\subfigure[Average subclass accuracy and class accuracy\label{subfig: avg_sub}]{
\includegraphics[width=.45\columnwidth]{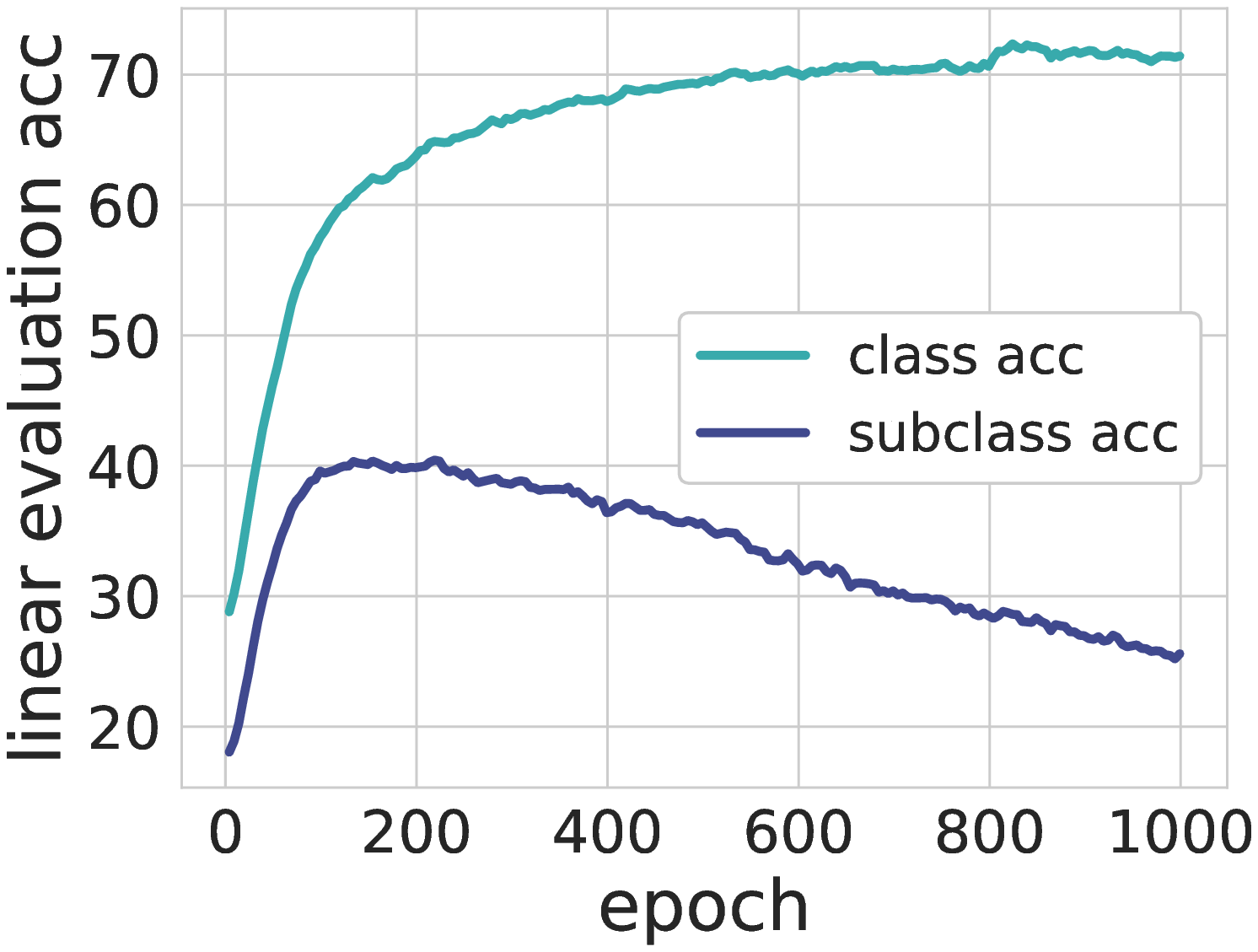}
}
\hspace{.2cm}
\subfigure[Subclasses are initially learned well but later unlearned\label{subfig: some_subclasses}]{
\includegraphics[width=.45\columnwidth]{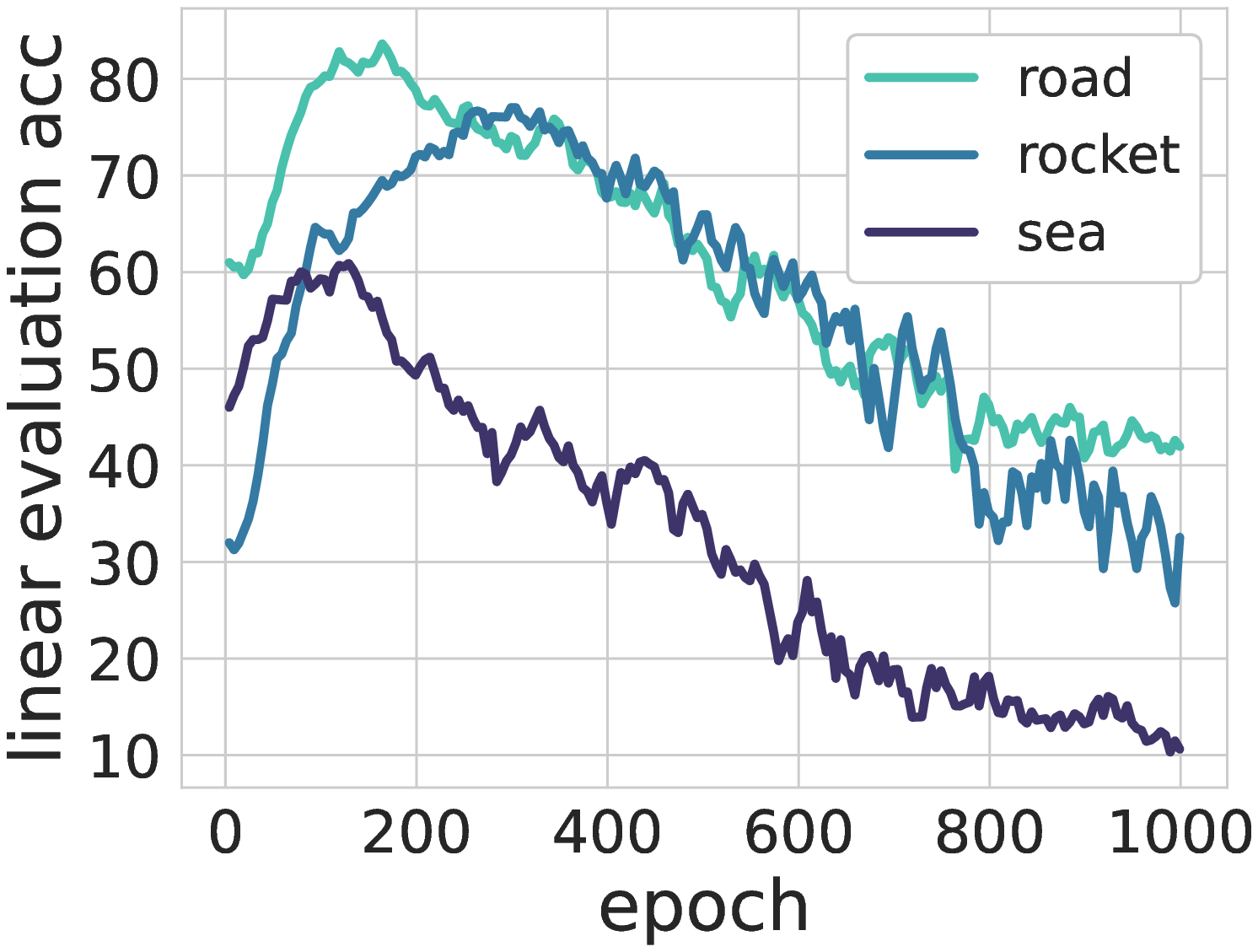}
}
\vspace{-.2cm}
    \caption{(a) Average subclass accuracy and class accuracy. (b) Accuracy in subclasses `road', `rocket' and `sea'. In both plots, the subclass accuracy increases and then decreases, which confirms that subclasses are learned early in training before class collapse happens. The class accuracy only increases during training. }
    \label{fig: exp_early_in_training}
    \vspace{-3mm}
\end{figure}

More importantly, the same phenomenon can be observed in \textit{neural networks}. We use SGD to train a ResNet18 \cite{he2016deep} on CIFAR-100 \cite{krizhevsky2009learning} with supervised CL loss \cite{khosla2020supervised} with 20 class (superclass) labels, and perform linear evaluation on embeddings of test data with 100 subclass (class) labels (see details in Appendix \ref{apdx:exp}). We observe that the subclass accuracy increases during the first 200 epochs before it starts to drop (Figure \ref{subfig: avg_sub}). Some subclasses can even achieve a  high accuracy around $80\%$ (Figure \ref{subfig: some_subclasses}). 
This is surprising as it confirms that models trained with commonly used loss functions \textit{do} learn subclass features early in training. 



\paragraph{Empirical Evidence Showing that Class Collapse Eventually Happens in (S)GD.} 

We simulate our theoretical analysis using numerical experiments to show that gradient descent converges to a minimizer that exhibits class collapse, despite learning subclasses early in training. We visualize the embeddings of test data at different epochs in Figure \ref{fig: cc_embedding}, and plot the alignment between weights and class/subclass features in Figure \ref{fig: cc_alignment}. Subclasses are perfectly separated and the weights align with both $\vct{v}_1$ and $\vct{v}_2$ after around 100 epochs of training. The model then starts unlearning $\vct{v}_2$ which causes the alignment to drop, thus subclasses are merged in the embedding space.  We also confirm that same conclusion holds for neural networks in realistic settings. In Figure \ref{fig: exp_early_in_training}, we see that the subclass accuracy drops after around 200 epochs of training and eventually reaches a low value. In contrast, the class accuracy does not drop during training.\looseness=-1

\paragraph{Minimum Norm Minimizer Exhibits Class Collapse.}
Note that in Theorem \ref{thm: early_in_training}, assuming $\mu\neq 0$ leads us to discovering that subclasses are learned early in training. Here, we extend Theorem \ref{thm: minimizer_with_cc} to this setting under asymptotic class collapse. \looseness=-1

\begin{definition}[Asymptotic Class Collapse]\label{def: asymp_cc} We say asymptotic class collapse happens when $\|\mtx{W}\vct{v}_2\|=O(\frac{\sigma_\xi}{\sqrt{mn}}) = o(1)$.\looseness=-1
\end{definition}

This definition implies that: (1) representations of subclasses are not well separated, hence it is nearly impossible to distinguish between them, and (2) the distinguishability of subclasses is at odds with generalization, which improves as number of augmented views per example $m$ and size of training data $n$ increase. Thus, while this definition is a relaxation of Definition \ref{def: exact_cc}, practically, this results in equally severe class collapse. \looseness=-1 
\begin{theorem}[Extension of Theorem \ref{thm: minimizer_with_cc} for $\mu_2 \neq 0$]\label{thm: min_norm_asymp}
Let $\mtx{\Theta}^{**}\!=\![\mtx{W}^{**}~ \vct{b}^{**}]$ be the minimum norm minimizer of $\mathcal{L}_\supcon$: 
$$
\mtx{\Theta}^{**} = \arg\min_{\mtx{\Theta^*}} \| \mtx{\Theta}^* \|_F ~s.t.~ \mtx{\Theta}^*\in\arg\min_{\mtx{\Theta}} \mathcal{L}_\supcon(\mtx{\Theta}).
$$
Then with probability at least $1-O(\frac{m^2n^2}{d})=1-o(1)$, asymptotic class collapse happens, i.e.,  $$\|\mtx{W}^{**} \vct{v}_2 \|=O(\frac{\sigma_\xi}{\sqrt{mn}})=o(1).$$
\end{theorem}
\vspace{-4mm}


\subsection{Simplicity Bias of (S)GD}

We reiterate our main findings:
\vspace{-2mm}
\begin{enumerate}
    \item Minimizing the supervised contrastive loss \textit{does not} necessarily lead to class collapse.
    \vspace{-2mm}
    \item However, \textit{simpler} minimizers of the supervised contrastive loss (e.g. \textit{minimum norm}) do suffer from class collapse.\looseness=-1
    \vspace{-2mm}
    \item Optimizing with (S)GD does learn the subclass features early in training, but eventually unlearns them, resulting in class collapse.
\end{enumerate}
\vspace{-2mm}
 These coupled with the fact that (S)GD is known to have a bias towards simpler solutions \cite{kalimeris2019sgd} prompt us to conjecture:
\begin{center}
\emph{The simplicity bias of (S)GD leads it to unlearn subclass features, thus causing class collapse}.
\end{center}


The simplicity bias of (S)GD has not been rigorously studied for CL, and our results indicate the surprising role it may play in class collapse. Note that, the supervised contrastive loss is different than common supervised objectives, where the role of such bias of (S)GD is understood better \cite{gunasekar2018implicit,soudry2018implicit,ji2019implicit,wu2019towards,lyu2021gradient}. Rather, the supervised CL objective can 
be reformulated as a matrix factorization objective (Eq. \ref{eq:matrix_fac}), where the debate on the bias of (S)GD (e.g., minimum norm \cite{gunasekar2017implicit} or rank \cite{arora2019implicit,razin2020implicit}) is still ongoing. \looseness=-1






\section{Understanding Feature Suppression in Unsupervised CL}


Empirically, feature suppression can be observed due to a variety of reasons \cite{addressing_feature_suppression_2020,chen2021intriguing,shortcut2021_robinson}. Easy features for unsupervised CL are those that allow the model to discriminate between examples (highly discriminative). Here, we consider different ways irrelevant features can be easy (highly discriminative) and characterize how this can lead to feature suppression. We show that the types of feature suppression we consider can be largely attributed to insufficient embedding dimensionality and/or poor data augmentations. Surprisingly, we find again that the minimum norm simplicity bias is critical in explaining this phenomenon.

\looseness=-1
 

\subsection{Feature Suppression due to Easy Irrelevant Features and Limited Embedding Space}

In Theorem \ref{thm: FS_1}, we show that easy (discriminative) irrelevant features can suppress the class feature when the embedding dimensionality is limited. 
For clarity, we let $\mu_2=0$. 
\looseness=-1




\begin{theorem}[Feature Suppression 1]\label{thm: FS_1}
Assume $p\leq K$. Let $L$ be the $(K+1)$-element tuple $\big[1, \phi_1^2,
\phi_2^2, \frac{\phi_3^2}{K-2}, \dots, \frac{\phi_K^2}{K-2} \big]$ whose last $K$ elements are the variances of features. If $\phi_1^2$ is not among the $p$ largest elements in $L$, then with probability at least $1-O(\frac{m^2n^2}{d})=1-o(1)$: (1) there exists a global minimizer $\mtx{\Theta}^*$ of $\mathcal{L}_\unsup$ such that $\| \vct{W}^* \vct{v}_1\| = \Omega(1)$, (2) However, the minimum norm minimizer $\mtx{\Theta}^{**}$ satisfies $\| \vct{W}^{**} \vct{v}_1\| = 0$.
\end{theorem}

We prove the theorem in Appendix \ref{apdx: FS}. The elements except the first one in tuple $L$ can be interpreted as the variance of examples at each coordinate $\vct{v}_k, k=1,2,\dots, K$, which indicates how much the examples are discriminated by each feature.
The theorem shows that when the embedding space is not large enough to represent all the $K$ features (which requires $K+1$ dimensions), the minimum norm minimizer only picks the most discriminative ones. In practice, the embedding space in unsupervised CL is relatively low-dimensional (compared to input dimensionality) and thus the model cannot fit all the information about inputs into the embedding space. As is suggested by Theorem \ref{thm: FS_1}, if the training algorithm prefers functions with certain simple structures, only the easiest (most discriminative) features that can be mapped into the embedding space by less complex functions (e.g., smaller norm) are learned. The class features are suppressed if they are not amongst the easiest ones. \looseness=-1

\begin{remark}
 Following the same analysis we can also show that when $\phi_1$ is among the $p$ largest elements in $L$, i.e., the class feature is among the easiest (most discriminative) ones, the class feature $\vct{v}_1$ is learned by the minimum norm minimizer; when $\phi_1$ is exactly on par with some other element as the $p$-th largest, there exist both minimum norm minimizers that learn and do not learn the class feature $\vct{v}_1$.
\end{remark}

\begin{figure}[!t]
    \centering
 \subfigure[$K=3, p=3$]{
\includegraphics[width=0.4\columnwidth]{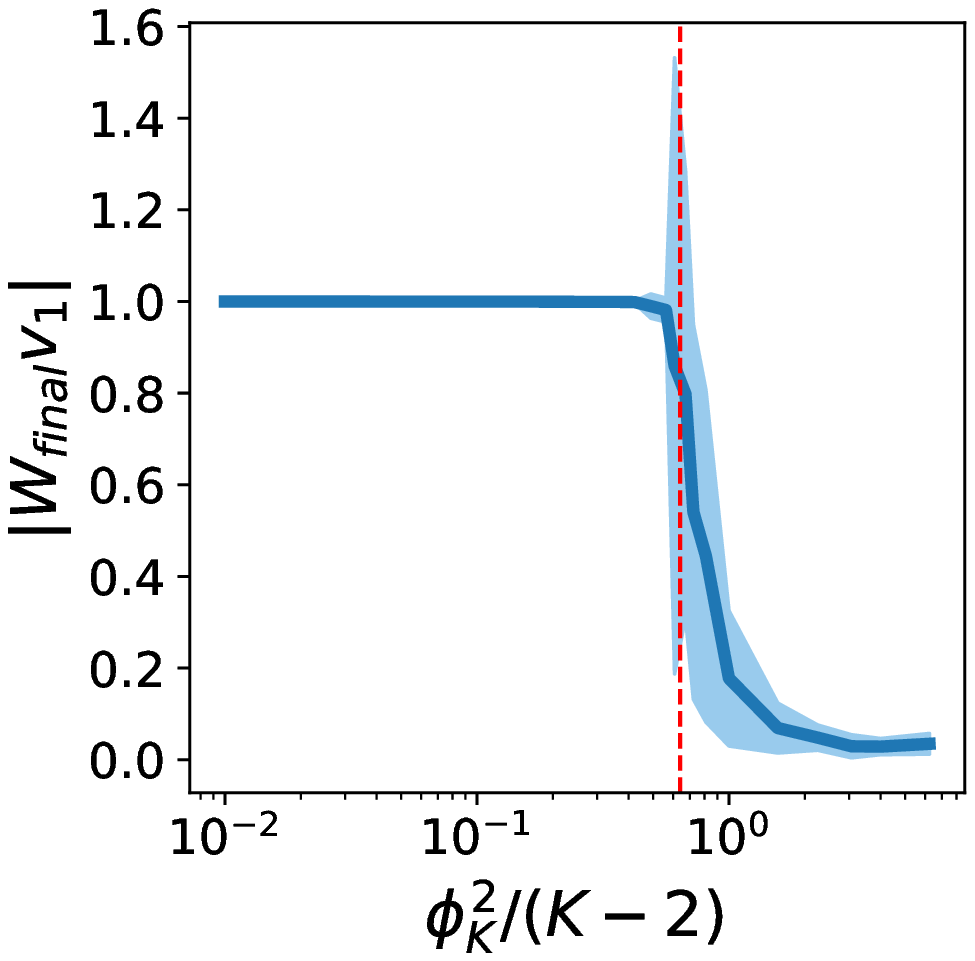}
}
 \subfigure[$K=50, p=50$]{
\includegraphics[width=0.4\columnwidth]{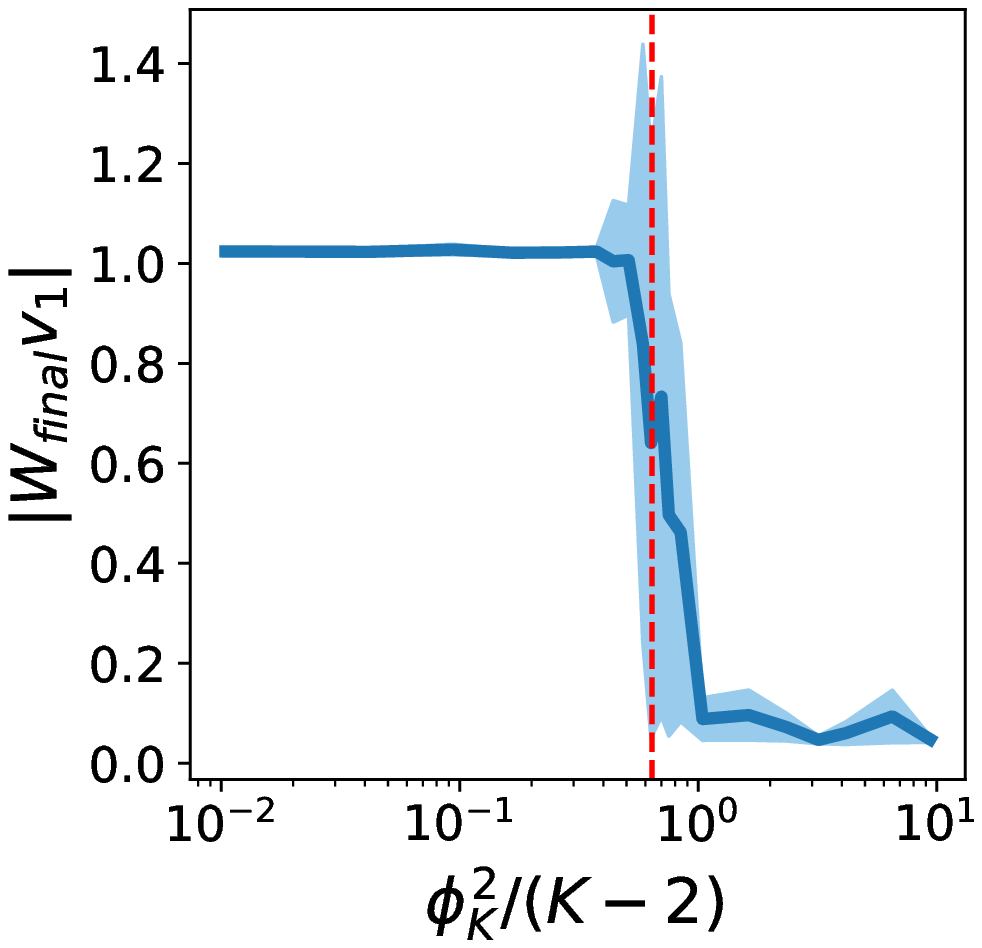}
}
\vspace{-.2cm}
    \caption{{\bf The irrelevant feature suppresses the class feature when its variance is beyond the variance of the class feature (the red vertical line).} We let $d=2000, p=K$, $\phi_1=0.8, \phi_2=1, \mu=0$, $\frac{\phi_k^2}{K-2}>\phi_1, \forall k\in [K-1]$ and vary $\phi_K$. Thus whether $\phi_1^2$ is among the $p$ largest variances only depends on $\phi_K$. We train the linear model to convergence. Plots show that the alignment between the trained weights and $\vct{v}_1$ drops when $\phi_K$ increases. We report the average of 10 runs. The result diverges at $\frac{\phi_K^2}{K-2}= \phi_1^2$ indicating that the model can learn either $\vct{v}_1$ or $\vct{v}_K$ in this case.  \looseness=-1 }
    \label{fig: fs_alignment}
    \vspace{-.2cm}
\end{figure}

\begin{figure}[!t]
    \centering
    \includegraphics[width=0.5\columnwidth]{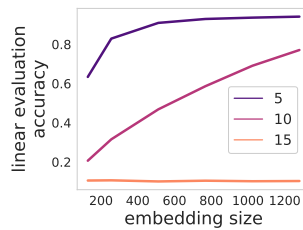}
    \label{fig: fs_embedding_space}
    \vspace{-.2cm}
    \caption{Effect of embedding size on feature suppression in MNIST RandBit\cite{chen2021intriguing}. Legends show the number of bits in the extra channel which indicates how easy (discriminative) the irrelevant features are. We observe that (1) increasing the easiness of irrelevant features exacerbates feature suppression; (2) increasing the embedding size alleviates feature suppression. \looseness=-1 }
    \vspace{-.4cm}
\end{figure}

\paragraph{Numerical Experiments with  GD.} 
Our theory for the minimum norm minimizer matches the experimental results for models trained with GD. We let $p=K$ 
and 
let $1\geq \phi_2^2\geq \frac{\phi_3^2}{K-2}\geq\dots\geq\frac{\phi_{K-1}^2}{K-2} > \phi_1^2$
so that $\phi_1^2$  must be among the smallest two variances i.e. $\vct{v}_1$ is among the two most difficult features. Then we vary $\phi_K$ and see how the trained weights align with $\vct{v}_1$. Consistent with Theorem 1, Figure \ref{fig: fs_alignment} shows that $\vct{v}_1$ is suppressed when $\frac{\phi_K^2}{K-2}> \phi_1^2$. Interestingly, we also see that the result at $\frac{\phi_K^2}{K-2}=\phi_1^2$ diverges, indicating that GD can find both minimizers that learn and do not learn $\vct{v}_1$ when the variances at $\vct{v}_1$ and $\vct{v}_K$ are the same.

\begin{table}[!t]
    \caption{Effect of embedding size on feature suppression in CIFAR-10/100 RandBit. `Acc' refers to class accuracy and `Sub Acc' refers to subclass accuracy. We see that increasing embedding size alleviates feature suppression, improving class/subclass accuracy. \looseness=-1}
    \label{tab:cifar_10_increasing_width}
    \vspace{1.5mm}
    \centering
    \begin{tabular}{|c |c c |c c|}
    \hline
       \multirow{2}{*}{$w$} & \multicolumn{2}{|c|}{CIFAR-10 RandBit} & \multicolumn{2}{|c|}{CIFAR-100 RandBit} \\
     & Sub Acc & Acc & Sub Acc & Acc \\
    \hline
        4 & 34.38 & 86.73 & 11.67 & 23.53 \\
        64 & 71.96 & 96.82 & 34.11 & 52.32 \\
        128 & 76.69 & 97.65 & 38.51 & 57.40 \\
    \hline
    \end{tabular}
    \vspace{-4mm}
\end{table}

\textbf{Empirically Verifying Benefits of Larger Embedding Size.}
Theorem \ref{thm: FS_1} also provides one practical solution for feature suppression due to limited embedding size:
increasing the embedding size so that every feature can be learned by the model. To provide empirical evidence for this, we conduct two sets of experiments: 
    
    First, we train 5-layer convolutional networks on the RandomBit dataset with the same setup as in \cite{chen2021intriguing}, but we vary the embedding size (see details in Appendix \ref{apdx:exp}). Varying the \# bits in the extra channel intuitively controls how discriminative the irrelevant feature are, i.e., how easy-to-learn it is for CL. In this setting, the random bit can suppress the MNIST digits. We make two observations in Figure \ref{fig: fs_embedding_space}: (1) with a fixed embedding size, increasing easiness (number of random bits) of the irrelevant features exacerbates feature suppression; (2) with a fixed easiness of irrelevant features, increasing the embedding size alleviates feature suppression. \looseness=-1
    
    Second, we train ResNet18 \cite{he2016deep} on the CIFAR-10/100 RandBit Dataset, constructed similarly to the MNIST RandBit dataset but with images from CIFAR-10/100 \cite{krizhevsky2009learning} (see Appendix \ref{apdx:datasets}). For CIFAR-10, we use 2 random bits, and for CIFAR-100, we use one random bit as the class irrelevant features. Table \ref{tab:cifar_10_increasing_width} presents the test performance for different values of the model width $w$, where a larger $w$ indicates a larger embedding size (see Appendix \ref{apdx:additional_embedding_size} for details). On both datasets, increasing the embedding size alleviates feature suppression, leading to improvements in both class and subclass accuracies. We also provide additional experiments and discussion in Appendix \ref{apdx:additional_embedding_size}. \looseness=-1
Both experimental results confirm the conclusion drawn from the theoretical analysis.\looseness=-1



\subsection{Feature Suppression due to High-dimensional Irrelevant Features and Imperfect Augmentation}\label{sec: fs_2}
Empirically, another form of feature suppression has been observed that cannot be remedied by larger embedding dimensionality \cite{addressing_feature_suppression_2020}. 
We characterize this form of feature suppression by defining easy irrelevant features as being: (1) drawn from a high dimensional space so that the collection of irrelevant features is large and discriminating based on irrelevant features is easier, (2) less altered by data augmentation compared to the class feature. 

For (1), formally we assume $K=\omega(n^2)$, as opposed to assumption \ref{assump: balanced} which implies that $K$ is smaller than $n$. A consequence of this assumption is that with high probability the $n$ original examples each have a unique irrelevant feature. For (2) we consider the following imperfect data augmentation:
\begin{definition}[Imperfect data augmentation $\badaug{\cdot}$]
For a given example $\vct{x}=\vct{\mu}+\vct{\xi}\in\hat{\mathcal{D}}_\orig$, \looseness=-1
\begin{align}
    \nonumber
    \badaug{\vct{x}} =& \vct{u} + \zeta'\vct{v}_1 + \zeta''\vct{v}_2+ \vct{\xi}',
\end{align}
where $\zeta'\sim\mathcal{N}(0, \sigma_\zeta'^2)$, $\zeta''\sim\mathcal{N}(0, \sigma_\zeta''^2)$, $\sigma_\zeta'^2,\sigma_\zeta''^2 \neq 0$ and $\vct{\xi}'$ is a new random variable drawn from $\mathcal{N}(\vct{\xi}, \mtx{\Sigma}_{\vct{\xi}})$ with $\text{rank}{(\mtx{\Sigma}_{\vct{\xi}})} \leq  \frac{m}{2}$.
\end{definition}
In the definition, the data augmentation adds small perturbations ($\zeta'$ and $\zeta''$) to class and subclass features, weakly alters the noise, but preserves the irrelevant features. For example, on Colorful-Moving-MNIST \cite{tian2020makes} 
constructed by assigning each MNIST digit a background object image selected randomly from STL-10, the colorful background objects are high-dimensional and the colors are invariant to data augmentations without color distortion. \looseness=-1




\begin{theorem}[Feature Suppression 2]\label{thm: FS_2}
If $K = \omega(n^2)$ and augmentation is $\badaug{\cdot}$, with probability $\geq 1-o(\frac{n^2m^2}{d} + \frac{1}{n}) = 1-o(1)$, the minimum norm minimizer $\mtx{\Theta}^*=[\mtx{W}^*, \vct{b}^*]$ satisfies $\|\mtx{W}^* \vct{v}_1\| = 0$.
\end{theorem}
This theorem shows that feature suppression can happen even when embedding dimensionality $p$ is arbitrarily large and helps understand empirical observations made both in our work (Figure \ref{fig: fs_embedding_space}, the line with 15 bits) and previous work. 
For example \citet{addressing_feature_suppression_2020} showed that on Colorful-Moving-MNIST, the colorful background can suppress learning the digits especially when color distortion is not used in augmentation, and increasing embedding size does not address the issue. 

In conclusion, Theorem \ref{thm: FS_2} highlights that designing data augmentations that disrupt the highly-discriminative irrelevant features is a key to addressing feature suppression.

\section{Combining Supervised and Unsupervised CL Losses Can Avoid Both Class Collapse and Feature Suppression}


We now consider the following loss which is a weighted sum of the supervised and unsupervised CL loss functions: \looseness=-1
\begin{align}
    \nonumber
    \mathcal{L}_{\intrpl, \beta} (\mtx{\Theta})= \beta \mathcal{L}_\supcon(\mtx{\Theta}) + (1-\beta)\mathcal{L}_\unsup(\mtx{\Theta}).
\end{align}
Similar loss functions have been proposed recently with notable empirical success. For example, \citet{cc_chen2022perfectlybalanced} put forth a weighted sum of supervised CL loss and class-conditional InfoNCE (which has similar effect as $\mathcal{L}_\unsup$ in our setting) to avoid class collapse. \citet{islam2021broad} empirically observed that the joint objective of supervised and unsupervised contrastive loss leads to better transferability of the learned models than their supervised counterparts. However, we still lack a theoretical understanding of why this weighted sum of losses can outperform both losses. 


From our investigation of class collapse and feature suppression, the benefit of the joint objective $\mathcal{L}_\intrpl$ becomes evident: the unsupervised term in $\mathcal{L}_\intrpl$ increases the chance of learning features that do not appear relevant to the labels but might be useful for downstream tasks, while the supervised term in $\mathcal{L}_\intrpl$ ensures that even hard-to-learn class features are learnt.  Thus, $\mathcal{L}_\intrpl$ can learn rich representations capturing more task relevant information than either $\mathcal{L}_\unsup(\mtx{\Theta})$ or $\mathcal{L}_\supcon(\mtx{\Theta})$. We show below that with an appropriate choice of $\beta$, $\mathcal{L}_\intrpl$ can provably succeed where $\mathcal{L}_\supcon$ fails due to collapse and $\mathcal{L}_\unsup$ fails due to feature suppression (for clarity, we let $\mu=0$). \looseness=-1

\begin{theorem}\label{thm: joint}
W.L.O.G., assume $\phi_3 \geq \phi_4 \geq \dots \geq \phi_K $. If $p\leq K$, $\phi_2^2 > \frac{\phi_{p-2}^2}{K-2} $ and $\phi_1^2 < \frac{\phi_{p-1}^2}{K-2}$, then by Theorem \ref{thm: minimizer_with_cc} the minimum norm minimizer of $\mathcal{L}_\supcon$ suffers from class collapse and by Theorem \ref{thm: FS_1} the minimum norm minimizer of $\mathcal{L}_\unsup$ suffers from feature suppresion. However, for constant $\beta \in (0, 1)$, the minimum norm minimizer of $\mathcal{L}_{\intrpl,\beta}$, denoted by $\mtx{\Theta}^*=[\mtx{W}^*~~ \vct{b}^*]$, satisfies $\|\mtx{W}^* \vct{v}_1 \|=\Omega(1)$ and $\|\mtx{W}^* \vct{v}_2 \|=\Omega(1)$. 
\end{theorem}

\begin{table}[!t]
    \caption{Joint loss alleviates class collapse on CIFAR-100. \looseness=-1}
    \label{tab:joint_cifar100}
    \centering
    \begin{tabular}{|c|c|}
    \hline
       Loss  & Subclass Acc  \\
    \hline
       SCL  & 26.11  \\
    \hline
    Joint loss ($\beta=0.8$) & 41.37 \\
    \hline
    \end{tabular}
\vspace{-6mm}
\end{table}

\begin{table}[!t]
    \caption{Joint loss alleviates feature suppresion on MNIST RandBit.\looseness=-1}
    \label{tab:joint_mnist}
    \centering
    \begin{tabular}{|c|c|}
    \hline
    Loss & Class Acc \\
    \hline
       UCL  & 61.21  \\
    \hline
      Joint loss ($\beta=0.5$)   & 79.37\\
    \hline
    \end{tabular}
    \vspace{-6mm}
\end{table}

\begin{table}[!t]
    \caption{Joint loss alleviates both class collapse and feature suppresion on CIFAR-100 RandBit.\looseness=-1}
    \label{tab:joint_cifar100_rand}
    \centering
    \begin{tabular}{|c|c|c|}
    \hline
      Loss   & Subclass Acc & Class Acc \\
    \hline
        SCL & 28.13 & 61.10 \\
    \hline
    UCL & 34.11 & 52.32 \\
    \hline
    Joint loss ($\beta=0.8$) & 35.72 & 63.94 \\
    \hline
    \end{tabular}
    \vspace{-4mm}
\end{table}

\textbf{Empirically Verifying Benefits of the Joint Loss.} We empirically examine the impact of the joint loss on MNIST RandBit, CIFAR-100, and CIFAR-100 RandBit. The training details are in Appenidx \ref{apdx:training_details}. The results indicate that the joint loss significantly improves performance in scenarios where SCL suffers from class collapse (Table \ref{tab:joint_cifar100}) and UCL suffers from feature suppression (Table \ref{tab:joint_mnist}). Furthermore, on CIFAR-100 RandBit dataset, where both phenomena can occur simultaneously, the joint loss effectively alleviates both issues (Table \ref{tab:joint_cifar100_rand}).\looseness=-1







\section{Discussion}

\textbf{Negative Impact of Simplicity Bias in Deep Learning.} The simplicity bias of optimization algorithms has been studied as a key beneficial factor in achieving good generalization \cite{gunasekar2017implicit,gunasekar2018implicit,soudry2018implicit,ji2019implicit,wu2019towards,lyu2021gradient}. However, our study reveals the negative impact of simplicity bias in CL.  In fact, it has also been conjectured to lead to undesirable outcomes in other scenarios, such as learning spurious correlations \cite{sagawa2020investigation} and shortcut solutions \cite{shortcut2021_robinson}.  We hope our study can inspire further theoretical characterization of the negative role of simplicity bias in these scenarios, thereby deepening our understanding and fostering potential solutions.\looseness=-1

\textbf{Connection to Neural Collapse.} Neural collapse (NC) \cite{papyan2020prevalence} refers to the collapse of representations within each class in supervised learning. 
Similar to the rationale in this study, overparameterized models that exhibit NC on training data can demonstrate different behaviors on test data due to their capacity to implement training set NC in various ways, and it is worth considering whether current theoretical frameworks \cite{han2021neural,zhu2021geometric,zhou2022all,zhou2022optimization,lu2022neural,fang2021exploring} can effectively capture NC on test data. In fact, the empirical results in \cite{hui2022limitations} emphasize the distinction between NC on training and test data, as there can be an inverse correlation between the two.
Our results suggest that analyzing the learned features and considering the inductive bias of training algorithms can aid in this distinction. 

\textbf{Theoretical Characterization of Class Collapse in (S)GD.} The results in Section \ref{sec: GD} highlight the need for theoretical characterization of class collapse in (S)GD. We provide two potential approaches  for future investigation. (1) Given that the objective can be reformulated as matrix factorization  (Eq. \ref{eq:matrix_fac}), and our Theorems \ref{thm: minimizer_with_cc} and \ref{thm: min_norm_asymp} on minimum norm minimizer, it is reasonable to investigate whether the implicit bias of (S)GD is to seek the minimum norm solution. We note that understanding the implicit bias in matrix factorization is a longstanding pursuit in the machine learning community, with no consensus reached thus far (see Appendix \ref{apdx:matrix_fac}). Hence, further effort is still needed. (2) As elaborated in Appendix \ref{apdx:two_terms}, the gradient consists of two terms with distinct roles. One promotes alignment with the subclass feature, while the other counteracts its influence. The relative scale of these two terms undergoes a phase transition (Figure \ref{fig:norm_grad}), and analyzing this can provide insights into class collapse.\looseness=-1




\section{Conclusion}


To conclude, we present the first theoretically rigorous characterization of the failure modes of CL: class collapse and feature suppression at test time. We explicitly construct minimizers of supervised contrastive loss to show that optimizing this loss does not necessarily lead to class collapse. Then we show that the minimum norm minimizer does exhibit class collapse. Our analysis also reveals a peculiar phenomenon for supervised CL, when optimized with (S)GD: subclass features are learned early in training and then unlearned. To analyze feature suppression, we consider two formalisms of easy features that can prevent learning of class features and provably attribute feature suppression to insufficient embedding space and/or imperfect data augmentations; thus, motivating practical solutions to this problem. The unified framework we develop to determine which features are learnt by CL allows us to also offer the only theoretical justification for recent empirical proposals to combine unsupervised and supervised contrastive losses. Perhaps, most surprisingly, our findings from this theoretical study indicate that simplicity bias of (S)GD is likely the driving factor behind class collapse and feature suppression. 

\textbf{Acknowledgment.}
This research was supported by 
the National Science Foundation CAREER
Award 2146492.\looseness=-1 

\bibliography{reference}
\bibliographystyle{icml2023}

\newpage
\appendix
\onecolumn

\section{Preliminaries}

\subsection{Effective dataset}
Analyzing training a linear model with bias on the data is equivalent to analyzing trainng a linear model without bias on: $\{ \cat{1}{\vct{x}_i}: \vct{x}_i \in \mathcal{D}_\aug \}$. Equivalently we can consider a dataset distribution where 
    \begin{align}
    \nonumber
    \vct{x} =& \vct{u} + \vct{\xi}, \\
    \nonumber
    \text{where}~~ \vct{u} =& \vct{v}_0 + y \phi_1\vct{v}_1 + (y_\sub \phi_2+ \mu)\vct{v}_2 + \rho\phi_{k} \vct{v}_{k}.
\end{align}
The definitions are identical to the one in Section \ref{sec: data_dist} except that each data now is in $\mathbbm{R}^{d+1}$ and has one constant feature $\vct{v}_0$ orthogonal to other $\vct{v}$'s. We train a linear model $f(\vct{x}) = \mtx{W}\vct{x}$ on such data. The definition of other notations such as  $\Daug$ in the following analysis are also adapted to this dataset accordingly. Other notations such as $\Daug$ in the subsequent analysis are adjusted accordingly to accommodate this dataset.

\subsection{Loss functions}

The loss functions can be rewritten as follows
\begin{align}
     \mathcal{L}_\supcon = & -2\hat{\mathbbm{E}}_{i\in[n],\vct{x}\in\mathcal{A}_i, \vct{x}^+\in\mathcal{A}_i} \left[ \vct{x}^{\top} \mtx{W}^{\top} \mtx{W} \vct{x}^+   \right]
    +\hat{\mathbbm{E}}_{\vct{x}\in\Daug, \vct{x}^-\in\Daug} \left[ \left( \vct{x}^{\top} \mtx{W}^{\top} \mtx{W} \vct{x}^-  \right)^2 \right] \\
    \nonumber
    = & -\Tr(2\mtx{M}^+\mtx{W}\mtx{W}^\top)+\Tr(\mtx{M}\mtx{W}^\top\mtx{W}\mtx{M}\mtx{W}^\top\mtx{W})\\
    \mathcal{L}_\unsup = & -2\hat{\mathbbm{E}}_{c\in\{-1,1\}, \vct{x}\in\mathcal{S}_c, \vct{x}^+\in\mathcal{S}_c} \left[ \vct{x}^{\top} \mtx{W}^{\top} \mtx{W} \vct{x}^+   \right]
+\hat{\mathbbm{E}}_{\vct{x}\in\Daug, \vct{x}^-\in\Daug} \left[ \left( \vct{x}^{\top} \mtx{W}^{\top} \mtx{W} \vct{x}^-  \right)^2 \right] \\
    \nonumber
    = & -\Tr(2\tilde{\mtx{M}}\mtx{W}\mtx{W}^\top)+\Tr(\mtx{M}\mtx{W}^\top\mtx{W}\mtx{M}\mtx{W}^\top\mtx{W})\\
    \mathcal{L}_\intrpl = & (1-\beta)\mathcal{L}_\supcon+\beta \mathcal{L}_\unsup \\
    \nonumber
    = & -\Tr(2\bar{\mtx{M}}\mtx{W}\mtx{W}^\top)+\Tr(\mtx{M}\mtx{W}^\top\mtx{W}\mtx{M}\mtx{W}^\top\mtx{W}),
\end{align}
where we define the following
\begin{definition}\label{def: Ms}
$\mtx{M}, \mtx{M}^+, \tilde{\mtx{M}}$ are the covariance matrices of training examples, class centers and augmentation centers, respectively
\begin{align}
    \nonumber
    \mtx{M} = &  \frac{1}{mn}\sum_{i=1}^{mn}\vct{x}_i\vct{x}_i^\top \\
    \nonumber
    \mtx{M}^+ =& \frac{1}{2}\sum_{c\in\{-1, 1\}} (\frac{2}{mn}\sum_{\vct{x}\in\mathcal{S}_c} \vct{x} )(\frac{2}{mn}\sum_{\vct{x}\in\mathcal{S}_c} \vct{x} )^\top\\
    \nonumber
    \tilde{\mtx{M}} = & \frac{1}{n} \sum_{i=1}^n (\frac{1}{m}\sum_{\vct{x}\in\mathcal{A}_i} \vct{x}) (\frac{1}{m}\sum_{\vct{x}\in\mathcal{A}_i} \vct{x})^\top,
\end{align}
and
\begin{align}
    \nonumber
    \bar{\mtx{M}} = & (1-\beta)\mtx{M}^+ + \beta \tilde{\mtx{M}}.
\end{align}
\end{definition}

\section{Minimizers of Loss Functions}

We start with a technical lemma which we will need:
\begin{lemma} \label{lemma:diagonalizable_product}
    The product of two positive semidefinite matrices is diagonalizable.
\end{lemma}
Next, we present a lemma that facilitates the analysis of minimizers for various contrastive loss functions. To apply the lemma, simply substitute the respective covariance matrices ($\mtx{M}$, $\mtx{M}^+$, $\tilde{\mtx{M}}$) into $\mtx{P}$ and $\mtx{Q}$ as indicated.
\begin{lemma}\label{lemma: minimizers}
Let $\mtx{P}, \mtx{Q} \in \mathbbm{R}^{d+1}$ be positive semidefinite matrices such that $\colsp(\mtx{P}) \subset \colsp(\mtx{Q})$. Consider the function $\mathcal{L}: \mathbbm{R}^{p \times (d+1)} \to \mathbbm{R}$ given by
\begin{align}
    \mathcal{L}(\mtx{W}) &= \Tr[-2 \mtx{W}^{\top}\mtx{W} \mtx{P} + \mtx{W}^{\top}\mtx{W} \mtx{Q} \mtx{W}^{\top}\mtx{W}\mtx{Q}]
\end{align}
Then $\mtx{W}$ is a global minimizer of $\mathcal{L}$ if and only if
\begin{align}
    \nonumber
    \mtx{W}^\top \mtx{W}\mtx{Q}  = [\mtx{Q}^\dag \mtx{P}]_p
\end{align}
where notation $[\mtx{A}]_p$ represents the matrix composed of the first $p$ eigenvalues and eigenvectors of a positive semidefinite $\mtx{A}$ (if $p \geq \rank{\mtx{A}}$ then $[\mtx{A}]_p = \mtx{A}$).

Moreover, if $p \geq \rank(\mtx{P})$, then $\mtx{W}^{**}$ is a minimum norm global minimizer if and only if 
\begin{align}
    \nonumber
    \mtx{W}^{**\top} \mtx{W}^{**} = \mtx{Q}^\dag \mtx{P} \mtx{Q}^\dag
\end{align}
\end{lemma}

\begin{proof}
First consider points that satisfy the first order condition
\begin{align} \label{fonc1}
    \nabla_{\mtx{W}}(\mathcal{L}) &= -4\mtx{W}\mtx{P} + 4\mtx{W}\mtx{Q}\mtx{W}^{\top}\mtx{W}\mtx{Q} = 0
\end{align}
where $\nabla_{\mtx{W}}(\mathcal{L})$ is the matrix of partial derivatives of $\mathcal{L}$ with respect to each entry of $\mtx{W}$. 

Since $\mtx{Q}$ is positive semidefinite, it decomposes $\mathbbm{R}^{d+1}$ into the orthogonal direct sum $\ker(\mtx{Q}) \oplus \colsp(\mtx{Q})$. Observe that both subspaces are invariant under both $\mtx{P}$ and $\mtx{Q}$.

Now let $\vct{v} \in \colsp(\mtx{Q}) \cap \ker(\mtx{W}\mtx{Q})$. Note that $\mtx{P}\vct{v} \in \colsp(\mtx{Q})$, so $\mtx{P} v = \mtx{Q} \mtx{Q}^\dag \mtx{P} v$. Then from equation \ref{fonc1},
\begin{align} \label{fonc2}
    \vct{0} = (\mtx{W} \mtx{P} - \mtx{W} \mtx{Q} \mtx{W}^{\top} \mtx{W} \mtx{Q}) \vct{v} = \mtx{W} \mtx{Q}( \mtx{Q}^\dag \mtx{P} - \mtx{W}^{\top} \mtx{W} \mtx{Q}) \vct{v}
\end{align}

If in addition we assume $\vct{v} \in \ker(\mtx{W}\mtx{Q})$, then
\begin{align}
    \nonumber
    \vct{0} = \mtx{W} \mtx{Q} (\mtx{Q}^\dag \mtx{P} \vct{v})
\end{align}
namely $\mtx{Q}^\dag \mtx{P} \vct{v} \in \ker(\mtx{W}\mtx{Q})$. But $\colsp(\mtx{Q})$ is also $\mtx{Q}^\dag$-invariant, so $\mtx{Q}^\dag \mtx{P} \vct{v} \in \colsp(\mtx{Q})$. We conclude that $\ker(\mtx{W}\mtx{Q}) \cap \colsp(\mtx{Q})$ is $\mtx{Q}^\dag \mtx{P}$-invariant. Since $\mtx{Q}^{\dag}$ and $\mtx{P}$ are positive semidefinite, by \ref{lemma:diagonalizable_product} $\mtx{Q}^\dag \mtx{P}$ is diagonalizable. The only invariant subspaces of a diagonalizable operator are spans of its eigenvectors, so $\ker({\mtx{W}\mtx{Q}}) \cap \colsp(\mtx{Q})$ is the span of eigenvectors of $\mtx{Q}^\dag \mtx{P}$. 

Let $\colsp(\mtx{Q}) = \ker(\mtx{W}\mtx{Q}) \cap \colsp(\mtx{Q}) \oplus U$, where $U$ is the span of the remaining eigenvectors of $\mtx{Q}^\dag \mtx{P}$ in $\colsp(\mtx{Q})$. Then by equation \ref{fonc2}, $\mtx{W}^{\top} \mtx{W} \mtx{Q} = \mtx{Q}^\dag \mtx{P}$ on $U$.

Thus we have a basis $\vct{v}_1, \dots, \vct{v}_r, \dots, \vct{v_s}, \dots, \vct{v_d}$ s.t. $Span(\vct{v}_1, \dots, \vct{v}_r) = U, Span(\vct{v}_{r+1}, \dots, \vct{v}_s) = \ker(\mtx{W}\mtx{Q}) \cap \colsp(\mtx{Q}), Span(v_{s+1}, \dots, v_{d+1}) = \ker{\mtx{Q}}$, and in this basis
\begin{align*}
    \mtx{Q}^\dag \mtx{P} &= \begin{pmatrix}
    \lambda_1 \\
    & \ddots \\
    & & \lambda_r \\
    & & &  \lambda_{r+1} \\
    & & & & \ddots \\
    & & & & & \lambda_s \\
    & & & & & & 0 \\
    & & & & & & & \ddots \\
    & & & & & & & & 0
    \end{pmatrix} \\
    \mtx{W}^{\top} \mtx{W} \mtx{Q} &= \begin{pmatrix}
    \lambda_1 \\
    & \ddots \\
    & & \lambda_r \\
    & & & 0 \\
    & & & & \ddots \\
    & & & & & 0 \\
    & & & & & & 0 \\
    & & & & & & & \ddots \\
    & & & & & & & & 0
    \end{pmatrix} \\
    \mtx{W}^{\top} \mtx{W} \mtx{P} &= \begin{pmatrix}
    \lambda_1^2 \\
    & \ddots \\
    & & \lambda_r^2 \\
    & & & 0 \\
    & & & & \ddots \\
    & & & & & 0 \\
    & & & & & & 0 \\
    & & & & & & & \ddots \\
    & & & & & & & & 0
    \end{pmatrix}
\end{align*}
with $\lambda_1, \dots, \lambda_r, \dots, \lambda_q \neq 0$ for some $r \leq q \leq s$, where $r = \rank\mtx{W} \leq p, q = \rank(\mtx{P})$.

Then for all such $\mtx{W}$,
\begin{align*}
    \mathcal{L} &= \Tr[-2\mtx{W}^{\top} \mtx{W} \mtx{P} + \mtx{W}^{\top} \mtx{W} \mtx{Q} \mtx{W}^{\top} \mtx{W} \mtx{Q}] \\
    &= -2 \sum_{i = 1}^r \lambda_i^2 + \sum_{i = 1}^r \lambda_i^2 \\
    &= -\sum_{i = 1}^r \lambda_i^2
\end{align*}

It is clear from the above expression that the minimum among critical points is achieved if and only if 
\begin{align}
    \nonumber
    \mtx{W}^\top \mtx{W}\mtx{Q}  = [\mtx{Q}^\dag \mtx{P}]_p
\end{align}
(note that if matching anything beyond the qth eigenvalue is trivial since all such eigenvalues are zero).

It remains to check the behavior as $\|\mtx{W}\|_F$ grows large. Equivalently, $\mtx{W}^{\top} \mtx{W}$ has a large eigenvalue $\lambda$. Let $\vct{w}$ be a corresponding eigenvector. If $\vct{w} \in \ker{\mtx{Q}}$, then $\mtx{Q}\vct{w} = \mtx{P}\vct{w} = 0$, so we see that the loss is unchanged. Otherwise, $\vct{w}$ has some nonzero alignment with $\colsp(\mtx{W})$. But then $\Tr[\mtx{W}^{\top} \mtx{W} \mtx{Q} \mtx{W}^{\top} \mtx{W} \mtx{Q}]$ grows quadratically in $\lambda$, but $\Tr[-2\mtx{W}^{\top} \mtx{W} \mtx{P}]$ grows at most linearly in $\lambda$, hence the loss is large. We conclude that the previously found condition in fact specifies the global minimizers of $\mathcal{L}$.

From now on, assume that $p \geq q$. Then the global minimum is achieved if and only if
\begin{align} \label{min_norm_solution_sufficient_rank}
    \mtx{W}^\top \mtx{W}\mtx{Q}  = \mtx{Q}^\dag \mtx{P}
\end{align}

Let us now consider the minimum norm solution, i.e. the one that minimizes $\Tr(\mtx{W}^{\top} \mtx{W})$. Note that $\mtx{W}^{\top} \mtx{W}$ and $\mtx{Q}^\dag \mtx{P} \mtx{Q}^{\dag}$ are positive semidefinite. Let $\mathcal{B}$ be an orthonormal basis of eigenvectors for $\colsp(\mtx{Q}), \mathcal{C}$ an orthonormal basis for $\ker{\mtx{Q}}$. Then in the orthonormal basis $\mathcal{B} \cup \mathcal{C}$, we have the following block form of $\mtx{Q}^\dag \mtx{P} \mtx{Q}^{\dag}$
\begin{align}
    \mtx{Q}^\dag \mtx{P} \mtx{Q}^{\dag} &= \begin{pmatrix}
        \mtx{A} & \mtx{0} \\
        \mtx{0} & \mtx{0} 
    \end{pmatrix}
\end{align}
where $\mtx{A}$ is positive semidefinite.

Now equation \ref{min_norm_solution_sufficient_rank} implies that $\mtx{W} \mtx{W}^{\top}$ has the form
\begin{align}
    \mtx{W}^{\top} \mtx{W} &= \begin{pmatrix}
        \mtx{A} & \mtx{B} \\
        \mtx{B}^{\top} & \mtx{C} 
    \end{pmatrix}
\end{align}
where $\mtx{C}$ is also positive semidefinite matrix. Then $\|\mtx{W}\|_F = \Tr[\mtx{W}^{\top} \mtx{W}]$ is minimized exactly when $\Tr[\mtx{C}] = 0$. But this holds if and only if $\mtx{C} = \mtx{0}$. Now
suppose for the sake of contradiction $\mtx{B} \neq \mtx{0}$, say $b_{ij} \neq 0$ for some $i, j$. Then $\mtx{W}^{\top}\mtx{W}$ contains a submatrix
\begin{align}
    \begin{pmatrix}
        a_{ii} & b_{ij} \\
        b_{ij}       & 0 
    \end{pmatrix}
\end{align}
which has negative determinant. But this implies that $\mtx{W}^{\top}\mtx{W}$ is not positive semidefinite, a contradiction. We conclude that $\mtx{B} = \mtx{0}$ so that the minimum norm solution is precisely
\begin{align}
    \nonumber
    \mtx{W}^{**\top} \mtx{W}^{**} = \mtx{Q}^\dag \mtx{P} \mtx{Q}^\dag.
\end{align}
This completes the proof.

\end{proof}

\section{Some Properties of The Covariance Matrices}\label{apdx: prop_minimizers}
We assume $\frac{\sigma_\xi^2}{mn} = o(1)$.

With probability $\geq 1-O(\frac{m^2n^2}{d})$, we have that $\vct{\xi}_i^\top \vct{v}_k=0, ~\forall k, i$ and $\vct{\xi}_i^\top \vct{\xi}_j=0, ~\forall i,j$. The following discussion focuses on the properties of $\mtx{M}$, $\mtx{M}^+$, and $\tilde{\mtx{M}}$ when this condition is met.

 Write $\mtx{X} = \mtx{V}\cat{\vct{S}}{\sigma_\xi\mathbf{I}_{mn}}$ where $\mtx{V} = [\vct{v}_0, \vct{v}_1 ~\dots~ \vct{v}_K  \dots \vct{v}_{K+1} \dots \vct{v}_{mn+K} ]$ where $\vct{v}_{K+i}$ is the noise vector selected by example $\vct{x}_i$, and 
 \begin{align}
     \nonumber
     \mtx{S} = &
     \begin{bmatrix}
         1 & 1 & \dots & 1 \\
         y_1\phi_1 & y_2\phi_1 & \dots & y_{mn}\phi_1 \\
         \mu + y_{\sub ,1} \phi_2 & \mu + y_{\sub ,2} \phi_2  &\dots & y_{\sub, mn}\phi_2  \\
          \rho_1 \mathbbm{1}_{k_1=3}\phi_3 & \rho_2\mathbbm{1}_{k_2=3}\phi_3 & \dots &  \rho_{mn}\mathbbm{1}_{k_{mn}=3}\phi_3 \\
          \rho_1\mathbbm{1}_{k_1=4}\phi_4 & \rho_2\mathbbm{1}_{k_2=4}\phi_4 & \dots &  \rho_{mn}\mathbbm{1}_{k_{mn}=4}\phi_4 \\
          \vdots & \vdots & \ddots & \vdots \\
          \rho_1\mathbbm{1}_{k_1=K}\phi_K & \rho_2\mathbbm{1}_{k_2=K}\phi_K & \dots &  \rho_{mn}\mathbbm{1}_{k_{mn}=K}\phi_K\\
     \end{bmatrix} \\
     \label{eq: svd_S}
     = & \vct{S}' \bar{\vct{Y}},
 \end{align}
 where 
 \begin{align}
 \nonumber
     \mtx{S}' \coloneqq  \begin{bmatrix}
         \sqrt{mn} & 0 & 0 & 0  &\dots & 0 \\
         0 & \sqrt{mn}\phi_1 & 0 & 0 & \dots & 0 \\
         \sqrt{mn}\mu & 0 & \sqrt{mn}\phi_2 & 0 & \dots & 0 \\
         0 & 0 & 0 & \sqrt{\frac{mn}{K-2}}\phi_3  & \dots & 0\\
         \vdots & \vdots & \vdots & \vdots  & \ddots & \vdots\\
         0 & 0 & 0 & 0  & \dots & \sqrt{\frac{mn}{K-2}}\phi_K\\
     \end{bmatrix},
 \end{align}
 and 
 \begin{align}
     \nonumber
     \bar{\mtx{Y}} \coloneqq \begin{bmatrix}
        \frac{1}{\sqrt{mn}} & \frac{1}{\sqrt{mn}} & \dots & \frac{1}{\sqrt{mn}} \\
         y_1\frac{1}{\sqrt{mn}} & y_2\frac{1}{\sqrt{mn}} & \dots & y_{mn}\frac{1}{\sqrt{mn}} \\
         y_{\sub ,1}\frac{1}{\sqrt{mn}} &  y_{\sub ,2} \frac{1}{\sqrt{mn}}  &\dots & y_{\sub, mn}\frac{1}{\sqrt{mn}} \\
          \rho_1 \mathbbm{1}_{k_1=3}\sqrt{\frac{K-2}{mn}} & \rho_2\mathbbm{1}_{k_2=3}\sqrt{\frac{K-2}{mn}} & \dots &  \rho_{mn}\mathbbm{1}_{k_{mn}=3}\sqrt{\frac{K-2}{mn}}  \\
          \rho_1\mathbbm{1}_{k_1=4}\sqrt{\frac{K-2}{mn}}  & \rho_2\mathbbm{1}_{k_2=4}\sqrt{\frac{K-2}{mn}}  & \dots &  \rho_{mn}\mathbbm{1}_{k_{mn}=4}\sqrt{\frac{K}{mn}}  \\
          \vdots & \vdots & \ddots & \vdots \\
          \rho_1\mathbbm{1}_{k_1=K}\sqrt{\frac{K-2}{mn}}  & \rho_2\mathbbm{1}_{k_2=K}\sqrt{\frac{K-2}{mn}}  & \dots &  \rho_{mn}\mathbbm{1}_{k_{mn}=K}\sqrt{\frac{K-2}{mn}} \\
     \end{bmatrix}
 \end{align}
 It should be noted that the rows of $\bar{\mtx{Y}}$ are orthonormal due to the assumption of a balanced dataset. Consequently, to obtain the singular value decomposition (SVD) of $\mtx{S}$, it suffices to find the SVD of $\mtx{S}'= \mtx{P}'\mtx{\Lambda}'\mtx{Q}'^\top $. Moreover, the right singular vectors of $\mtx{S}$ with non-zero singular values are given by the rows of $\mtx{Q}'^\top\bar{\mtx{Y}}$.
 
We write $\mtx{M}$ as $\mtx{V} \vct{G} \mtx{V}^\top$ where $\mtx{G}$ is given by
 \begin{align}
    \nonumber
    \begin{bmatrix}
        \frac{1}{mn}\mtx{S}\mtx{S}^\top & \frac{\sigma_\xi}{mn}\mtx{S}\\
        \frac{\sigma_\xi}{mn}\mtx{S}^\top & \frac{\sigma_\xi^2}{mn}\mathbf{I}_{mn}
    \end{bmatrix}.
\end{align}

Now we are ready to show the following lemma which describes the SVD of $\mtx{G}$.
\begin{lemma}\label{lemma: SVD_perturb}
Let $\mtx{S}\in\mathbbm{R}^{K\times {nm}}$ be a rank-K matrix with SVD $\mtx{P}\mtx{\Lambda} \mtx{Q}^\top$, where $\mtx{P}\in\mathbbm{R}^{K\times K}, \mtx{\Lambda}\in\mathbbm{R}^{K\times mn}$ and $\mtx{Q}\in\mathbbm{R}^{mn\times mn}$. The $mn$ none-zero eigenvalues of the following matrix $\mtx{G}$
\begin{align}
    \nonumber
    \begin{bmatrix}
        \frac{1}{mn}\mtx{S}\mtx{S}^\top & \frac{\sigma_\xi}{mn}\mtx{S}\\
        \frac{\sigma_\xi}{mn}\mtx{S}^\top & \frac{\sigma_\xi^2}{mn}\mathbf{I}_{mn}
    \end{bmatrix}
\end{align}
are given by $\frac{\sigma_\xi^2}{mn}+\frac{\lambda_1^2}{mn}, \frac{\sigma_\xi^2}{mn}+\frac{\lambda_2^2}{mn}, \dots, \frac{\sigma_\xi^2}{mn}+\frac{\lambda_K^2}{mn}, \frac{\sigma_\xi^2}{mn}, \dots, \frac{\sigma_\xi^2}{mn} $, with the corresponding eigenvectors $\cat{\frac{1}{\sqrt{1+r_1^2}}\vct{p}_1}{\frac{r_1}{\sqrt{1+r_1^2}}\vct{q}_1}, \cat{\frac{1}{\sqrt{1+r_2^2}}\vct{p}_2}{\frac{r_2}{\sqrt{1+r_2^2}}\vct{q}_2}, \dots, \cat{\frac{1}{\sqrt{1+r_K^2}}\vct{p}_K}{\frac{r_K}{\sqrt{1+r_K^2}}\vct{q}_K}, \cat{\vct{0}_K}{\vct{q}_{K+1}}, \dots, \cat{\vct{0}_K}{\vct{q}_{mn}}$, where $r_k=\frac{\sigma_{\xi}}{\lambda_k}$.
\end{lemma}
\begin{proof}
Let 
$\begin{bmatrix}\mtx{P}\vct{a}\\ \mtx{Q}\vct{b}
\end{bmatrix}$ where $\vct{a}\in\mathbbm{R}^K$ and $\vct{b}\in\mathbbm{R}^{mn}$ be an eigenvector of $\mtx{G}$. By the definition of eigenvector there should exist $\alpha$ such that $\mtx{G}\cat{\mtx{P}\vct{a}}{\mtx{Q}\vct{b}} = \alpha \cat{\mtx{P}\vct{a}}{\mtx{Q}\vct{b}}$, i.e.,
\begin{align}
    \nonumber
\begin{cases}
    \frac{1}{mn}\mtx{P}\mtx{\Lambda}\mtx{\Lambda}^\top \vct{a} + \frac{\sigma_\xi}{mn}\mtx{P}\mtx{\Lambda}\vct{b} = \alpha\mtx{P}\vct{a} \\
    \frac{\sigma_\xi}{mn}\mtx{Q}\mtx{\Lambda}^\top \vct{a} + \frac{\sigma_\xi^2}{mn} \mtx{Q}\vct{b} = \alpha \mtx{Q}\vct{b},
\end{cases}
\end{align}
which reduces to 
\begin{align}
    \nonumber
    \begin{cases}
        (\alpha\mathbf{I}_K-\frac{1}{mn}\mtx{\Lambda}\mtx{\Lambda}^\top)\vct{a} = \frac{\sigma_\xi}{mn}\mtx{\Lambda}\vct{b} \\
        \frac{\sigma_\xi}{mn}\mtx{\Lambda}^\top \vct{a} = (\alpha - \frac{\sigma_\xi^2}{mn})\vct{b}.
    \end{cases}
\end{align}
Firstly, we observe that the rank of $\mathbf{G}$ is at most $mn$ because $\mtx{G}=\frac{1}{mn}\cat{\mtx{S}}{\sigma_\xi\mathbf{I}_{mn}}\cat{\mtx{S}}{\sigma_\xi\mathbf{I}_{mn}}^{\top}$. Then it is easy to check that the eigenvalues and eigenvectors in Lemma \ref{lemma: SVD_perturb} satisfy the above conditions and the eigenvectors are indeed orthonormal, which completes the proof.
\end{proof}
\begin{corollary}\label{corollary: v2_proj_ker}
The projection of $\vct{v}_2$ onto $\ker \mtx{M}$ has magnitude $\Theta(\frac{\sigma_\xi}{\sqrt{mn}})$.
\end{corollary}

\begin{corollary}\label{corollary: align_subclass}.
Assuming the dataset is balanced, then
\begin{align}
    \nonumber
    \sqrt{\vct{v}_2^\top \mtx{M}^\dag \mtx{M}^+ \mtx{M}^\dag \vct{v}_2} = \begin{cases}
        0, if~~ \mu=0 \\
        O(\frac{\sigma_\xi}{\sqrt{mn}}), if ~~\mu\neq 0 ~~\text{and}~~ \mu=\Theta(1).
    \end{cases}
\end{align}
\end{corollary}  
\begin{proof}
 Let $\mtx{L}\mtx{A}\mtx{L}^\top$ be the eigendecomposition of $\mtx{G}$. Then 
 $$
 \mtx{M}^\dag \vct{v}_2 = \mtx{V}\mtx{L}\mtx{A}^\dag \mtx{L}^\top 
 \begin{bmatrix}
  0 \\
  0 \\
  1 \\
  0 \\
  \vdots\\
  0
 \end{bmatrix}.
 $$
When $\mu=0$, we can express the SVD of $\mtx{S}$ (equation \ref{eq: svd_S}) and apply Lemma \ref{lemma: SVD_perturb} to obtain the following result.
 \begin{align}
     \nonumber
    \lambda_3 = \sqrt{mn}\phi_2 ,~    a_3 = \frac{\sigma_\xi^2}{mn}+\phi_2^2, ~~ r_3 = \frac{\sigma_\xi}{\sqrt{mn}\phi_2} \\
    \nonumber
    \vct{p}_k = \vct{e}_k, ~\forall k\in[K] ~~\text{and} ~~
    \vct{q}_3 =
        \begin{bmatrix}
        \frac{1}{\sqrt{mn}}y_{\sub, 1} \\
        \frac{1}{\sqrt{mn}}y_{\sub,2}\\
        \vdots\\
        \frac{1}{\sqrt{mn}}y_{\sub, mn}
    \end{bmatrix}, 
    ~~\text{and} ~~ \vct{l}_3 = \cat{\frac{1}{\sqrt{1+r_3^2}} \vct{p}_3}{\frac{r_3}{\sqrt{1+r_3^2}}\vct{q}_3}.
 \end{align}
Thus
 \begin{align}
     \nonumber
     \mtx{M}^\dag \vct{v}_2 = & \frac{1}{a_3\sqrt{1+r_3^2}}\mtx{V} \vct{l}_3 \\
     \nonumber 
     = & \frac{1}{a_3\sqrt{1+r_3^2}} (\frac{1}{\sqrt{1+r_3^2}}\vct{v}_2 +\frac{r_3}{\sqrt{1+r_3^2}}\sum_{i=1}^{mn}  \frac{1}{\sqrt{mn}} y_{\sub,i} \vct{v}_{K+i} ).
 \end{align}
 Let $\bar{\vct{x}}_{y}$ be the average of examples with label $y$ and let $\mathcal{S}_y$ collects indices of examples with label $y$. Then
 \begin{align}
     \bar{\vct{x}}_{y} = \vct{v}_0 + \vct{v}_1 + \frac{2\sigma_\xi}{mn} \sum_{i\in \mathcal{S}_y} \vct{v}_{K+i},
 \end{align}
 and
 \begin{align}
     \nonumber
     \bar{\vct{x}}_{y}^\top  \mtx{M}^\dag \vct{v}_2 =  \frac{r_3}{a_3  (1+r_3^2)} \frac{2\sigma_\xi}{mn}\sum_{i\in\mathcal{S}_y} \frac{1}{\sqrt{mn}} y_{\sub,i} =0.
 \end{align}
  Write $\mtx{M}^+$ as
 \begin{align}
     \nonumber
     \mtx{M}^+ =& \frac{1}{2} ( \bar{\vct{x}}_{+1}\bar{\vct{x}}_{+1}^\top+\bar{\vct{x}}_{-1}\bar{\vct{x}}_{-1}^\top ).
 \end{align}
 Then
 \begin{align}
     \nonumber
     \vct{v}_2^\top \mtx{M}^\dag \mtx{M}^+ \mtx{M}^\dag \vct{v}_2 = \frac{1}{2}( (\vct{v}_2^\top \mtx{M}^\dag \bar{\vct{x}}_{+1} )^2 +(\vct{v}_2^\top \mtx{M}^\dag \bar{\vct{x}}_{-1} )^2 ) = 0.
 \end{align}
 When $\mu\neq 0$, then there are at most two of $\vct{p}_k$'s that are not orthogonal to $\vct{e}_3$ (say $\vct{p}_1$ and $\vct{p}_3$). Additionally, all of their elements, except for the first one, are zero. The remaining corresponding quantities satisfy.
 \begin{align}
     \nonumber
     & \lambda_1, \lambda_3 = \Theta(\sqrt{mn}), \\
     \nonumber
     & a_1=\frac{\lambda_1^2}{mn}+\frac{\sigma_\xi^2}{mn},~~ a_3 =\frac{\lambda_3^2}{mn}+\frac{\sigma_\xi^2}{mn}& \\
     \nonumber
     & r_1 =\frac{\sigma_\xi}{\lambda_1}, r_3 =  \frac{\sigma_\xi}{\lambda_3},
 \end{align}
and $\vct{q}_1$ and $\vct{q}_3$ are just linear combinations of $\bar{\vct{y}}_\sub$ and $\frac{1}{\sqrt{mn}}\vct{1}$, where $\bar{\vct{y}}_\sub$ is a vector whose $i$-th element is $\frac{1}{\sqrt{mn}}y_{\sub, i}$. Then
\begin{align}
\nonumber
    \mtx{M}^\dag \vct{v}_2 = & \mtx{V}\left[ 
\frac{1}{a_1}\frac{1}{\sqrt{1+r_1^2}}c_{3,1} \vct{l}_1 +\frac{1}{a_3}\frac{1}{\sqrt{1+r_3^2}}c_{3,3} \vct{l}_3   \right] 
\end{align}
where $c_{i,j} = \vct{p}_{j}^\top\vct{e}_i$ are constants. For $i = 0,2$
\begin{align}
\nonumber
    \vct{v}_0^\top\mtx{M}^\dag \vct{v}_2 
= & \vct{e}_1^\top\left[ 
\frac{1}{a_1}\frac{1}{\sqrt{1+r_1^2}}c_{3,1} \vct{l}_1 +\frac{1}{a_3}\frac{1}{\sqrt{1+r_3^2}}c_{3,3} \vct{l}_3   \right] \\
\nonumber
= & \frac{1}{a_1}\frac{1}{1+r_1^2}c_{3,1}c_{1,1} +\frac{1}{a_3}\frac{1}{1+r_3^2}c_{3,3}c_{1,3}  \\
\nonumber
= & (\frac{mn}{\lambda_1^2} -\Theta(\frac{\sigma_\xi^2}{mn}))(1-\Theta(\frac{\sigma_\xi}{\lambda_1}))c_{3,1}c_{1,1} +  (\frac{mn}{\lambda_3^2} -\Theta(\frac{\sigma_\xi^2}{mn}))(1-\Theta(\frac{\sigma_\xi}{\lambda_3}))c_{3,3}c_{1,3} \\
\nonumber
= & 
\end{align}
where $|\epsilon_1|= O(\frac{\sigma_\xi}{\sqrt{mn}})$. Similarly,
\begin{align}
\nonumber
    \vct{v}_2^\top\mtx{M}^\dag \vct{v}_2 
= & \frac{mn}{\lambda_1^2}c_{3,1}c_{3,1} + \frac{mn}{\lambda_3^2}c_{3,3}c_{3,3} + \epsilon_2,
\end{align}
where $|\epsilon_2|=O(\frac{\sigma_\xi}{\sqrt{mn}})$. For $i > K$
\begin{align}
\nonumber
    \vct{v}_i^\top\mtx{M}^\dag \vct{v}_2 = &\vct{v}_i\mtx{V}\left[ 
\frac{1}{a_1}\frac{1}{1+r_1^2}c_{3,1} \vct{l}_1 +\frac{1}{a_3}\frac{1}{1+r_3^2}c_{3,3} \vct{l}_3   \right] \\
\nonumber
= &\vct{e}_i^\top\left[ 
\frac{1}{a_1}\frac{1}{1+r_1^2}c_{3,1} \vct{l}_1 +\frac{1}{a_3}\frac{1}{1+r_3^2}c_{3,3} \vct{l}_3   \right] \\
\nonumber
= & \epsilon_3,
\end{align}
where $|\epsilon_3| = O(\frac{\sigma_\xi}{mn})$. Additionally,
\begin{align}
\nonumber
    \bar{\vct{x}}_{y} = \vct{v}_0 + \vct{v}_1 +\mu\vct{v}_2 + \frac{2\sigma_\xi}{mn} \sum_{i\in \mathcal{S}_y} \vct{v}_{K+i}.
\end{align}
Then
\begin{align}
    \nonumber
    \bar{\vct{x}}_{y}^\top \mtx{M}^\dag \vct{v}_2 = \frac{mn}{\lambda_1^2}c_{3,1}c_{1,1} + \frac{mn}{\lambda_3^2}c_{3,3}c_{1,3} + \frac{mn}{\lambda_1^2}c_{3,1}c_{3,1} + \frac{mn}{\lambda_3^2}c_{3,3}c_{3,3} + O(\frac{\sigma_\xi}{\sqrt{mn}}) .
\end{align}
By straightforward calculation, we can verify that $\frac{mn}{\lambda_1^2}c_{3,1}c_{1,1} + \frac{mn}{\lambda_3^2}c_{3,3}c_{1,3} + \frac{mn}{\lambda_1^2}c_{3,1}c_{3,1} + \frac{mn}{\lambda_3^2}c_{3,3}c_{3,3}=0$. This equation can be equivalently examined as the satisfaction of the following condition:
\begin{align}
    \nonumber
    [1 ~~\mu] \begin{bmatrix}
        1 & \mu \\
        \mu & \mu^2+\phi_2^2
    \end{bmatrix}^{-1} \cat{0}{1} = 0.
\end{align}
Therefore $|\bar{\vct{x}}_{y}^\top \mtx{M}^\dag \vct{v}_2| =O(\frac{\sigma_\xi}{\sqrt{mn}})$, and consequently $\sqrt{\vct{v}_2^\top \mtx{M}^\dag \mtx{M}^+ \mtx{M}^\dag \vct{v}_2}=O(\frac{\sigma_\xi}{\sqrt{mn}})$.
\end{proof}
\begin{corollary}\label{corollary: irrelevant}
Similar to Corollary \ref{corollary: align_subclass}, we also have $\sqrt{\vct{v}_k^\top \mtx{M}^\dag \mtx{M}^+ \mtx{M}^\dag \vct{v}_k}=0$, $k=3,4,\dots, K$.
\end{corollary}

\begin{corollary}\label{corollary: v1_learned_supcon}
$\sqrt{\vct{v}_1^\top \mtx{M}^\dag \mtx{M}^+ \mtx{M}^\dag \vct{v}_1}=\Theta(1)$. It can be proved using the same strategy as in Corollary \ref{corollary: align_subclass}.
\end{corollary}

\begin{lemma}\label{lemma: Mtilde_Mdag}
(1) The first $K$ eigenvectors/eigenvalues of $\tmtx{M}$ match those of $\mtx{M}$. (2) $\mtx{M}^\dag\tilde{\mtx{M}}$ is identity on $\colsp (\tilde{\mtx{M}})$ and null on $\ker(\tilde{\mtx{M}})$, i.e.,  $\mtx{M}^\dag\tilde{\mtx{M}} = \tilde{\mtx{M}}^\dag\tilde{\mtx{M}} $.
\end{lemma}
\begin{proof}
We assign indices to the training examples such that the augmented examples from the same original example are indexed from $(l-1)\times m+1$ to $l\times m$, where $l$ ranges from 1 to $n$. Next, we define matrix $\tilde{\mtx{V}}=[\tilde{\vct{v}}_1, \tilde{\vct{v}}_2, \dots, \tilde{\vct{v}}_n]\in\mathbbm{R}^{d\times n}$, where
\begin{align}
    \nonumber
    \tilde{\vct{v}_i} =& \vct{v}_i, ~~ \forall 1\leq i \leq K,\\
    \nonumber
     \tilde{\vct{v}_i} =& \frac{1}{\sqrt{m}}\sum_{j=1}^m \vct{v}_{K+(i-1)\times m+j}, ~~ \forall K+1\leq i\leq n.
\end{align}
In other words, $\tilde{\mtx{V}}$ can be written as
\begin{align}
    \nonumber
    \tilde{\mtx{V}} = \mtx{V} \mtx{T},
\end{align}
where 
\begin{align}
    \nonumber
    \mtx{T} = 
    \begin{bmatrix}
        \mtx{I}_K & \mtx{0}_{K\times n}\\
        \mtx{0}_{mn\times K} & 
        \begin{bmatrix}
            \frac{1}{\sqrt{m}} \mtx{1}_{m\times 1} & 0 & 0 &\dots &0\\
            0 & \frac{1}{\sqrt{m}} \mtx{1}_{m\times 1} & 0 & \dots &0 \\
             0 & 0& \frac{1}{\sqrt{m}} \mtx{1}_{m\times 1}  & \dots &0
        \end{bmatrix}
    \end{bmatrix}
\end{align}

Note that, by the definition of our augmentation, the center of augmentations of the $i$-th original example, i.e.,  $\tilde{\vct{x}}_i =\frac{1}{m}  \sum_{j=1}^m \vct{x}_{K+(i-1)\times m+j}$, can be considered as an example with the same features as $\vct{x}_i$ but with an added noise term of $\frac{\sigma_\xi}{\sqrt{m}} \tilde{\vct{v}}_i $. Therefore we can change the basis to $\tilde{\mtx{V}}$ and express $\tilde{\mtx{M}}$ as 
\begin{align}
    \nonumber
    \tilde{\mtx{M}} = \tilde{\mtx{V}} \tilde{\mtx{G}} \tilde{\mtx{V}}^\top,
\end{align}
where
\begin{align}
    \nonumber
    \tilde{\mtx{G}} = 
    \begin{bmatrix}
       \frac{1}{n} \tilde{\mtx{S}}\tilde{\mtx{S}}^\top & \frac{1}{n}\frac{\sigma_\xi}{\sqrt{m}}\tilde{\mtx{S}} \\
       \frac{1}{n}\frac{\sigma_\xi}{\sqrt{m}}\tilde{\mtx{S}}^\top & \frac{1}{n}\frac{\sigma_\xi^2}{m}\mtx{I}_n 
    \end{bmatrix}
\end{align}
and 
\begin{align}
\label{eq: tildeS_SVD}
    \tilde{\mtx{S}}
     = \tilde{\mtx{S}}'\bar{\mtx{Y}}_\orig,
\end{align}
where 
\begin{align}
\nonumber
    \tilde{\mtx{S}}' \coloneqq          \begin{bmatrix}
         \sqrt{n} & 0 & 0 & 0  &\dots & 0 \\
         0 & \sqrt{n}\phi_1 & 0 & 0 & \dots & 0 \\
         \sqrt{n}\mu & 0 & \sqrt{n}\phi_2 & 0 & \dots & 0 \\
         0 & 0 & 0 & \sqrt{\frac{n}{K-2}}\phi_3  & \dots & 0\\
         \vdots & \vdots & \vdots & \vdots  & \ddots & \vdots\\
         0 & 0 & 0 & 0  & \dots & \sqrt{\frac{n}{K-2}}\phi_K\\
     \end{bmatrix},
\end{align}
and 
\begin{align}
    \bar{\mtx{Y}}_\orig \coloneqq &      \begin{bmatrix}
        \frac{1}{\sqrt{n}} & \frac{1}{\sqrt{n}} & \dots & \frac{1}{\sqrt{n}} \\
         y_1\frac{1}{\sqrt{n}} & y_2\frac{1}{\sqrt{n}} & \dots & y_{n}\frac{1}{\sqrt{n}} \\
         y_{\sub ,1}\frac{1}{\sqrt{n}} &  y_{\sub ,2} \frac{1}{\sqrt{n}}  &\dots & y_{\sub, n}\frac{1}{\sqrt{n}} \\
          \rho_1 \mathbbm{1}_{k_1=3}\sqrt{\frac{K-2}{n}} & \rho_2\mathbbm{1}_{k_2=3}\sqrt{\frac{K-2}{n}} & \dots &  \rho_{n}\mathbbm{1}_{k_{n}=3}\sqrt{\frac{K-2}{n}}  \\
          \rho_1\mathbbm{1}_{k_1=4}\sqrt{\frac{K-2}{n}}  & \rho_2\mathbbm{1}_{k_2=4}\sqrt{\frac{K-2}{n}}  & \dots &  \rho_{n}\mathbbm{1}_{k_{n}=4}\sqrt{\frac{K}{n}}  \\
          \vdots & \vdots & \ddots & \vdots \\
          \rho_1\mathbbm{1}_{k_1=K}\sqrt{\frac{K-2}{n}}  & \rho_2\mathbbm{1}_{k_2=K}\sqrt{\frac{K-2}{n}}  & \dots &  \rho_{n}\mathbbm{1}_{k_{n}=K}\sqrt{\frac{K-2}{n}} \\
     \end{bmatrix}_\orig.
\end{align}
We note that we use the subscript `$\orig$' of a matrix to indicate that its elements represent the corresponding quantities on the original dataset (e.g., $y_i$ is the label of the $i$-th original example). Let $\tmtx{P}'\tmtx{\Lambda}'\tmtx{Q}'^\top$ be the SVD of $\tmtx{S}$. Similar to equation \ref{eq: svd_S}, we observe that $\tmtx{P}'\tmtx{\Lambda}'(\tilde{\mtx{Q}}'^\top\bar{\mtx{Y}}_\orig)$ serves as an eigendecomposition of $\tmtx{S}$.

Now we make the following observations:
 \begin{itemize}
     \item[1.]  By Lemma \ref{lemma: SVD_perturb} (with $\mtx{G}$ replaced by $\tmtx{G}$) and the fact that $\tmtx{\Lambda}'$ collects the eigenvalues of $\tmtx{S}$, the eigenvalues of $\tmtx{G}$ are $\frac{\sigma_\xi^2}{mn}+\frac{\lambda_1'^2}{n}, \frac{\sigma_\xi^2}{mn}+\frac{\lambda_2'^2}{n}, \dots, \frac{\sigma_\xi^2}{mn}+\frac{\lambda_K'^2}{n}, \frac{\sigma_\xi^2}{n}, \dots, \frac{\sigma_\xi^2}{n} $, which are also the eigenvalues of $\tmtx{M}$ because $\tmtx{V}$ has orthonormal columns. With the observation that $\tmtx{S}'=\frac{1}{\sqrt{m}}\mtx{S}'$ ($\mtx{S}'$ is defined in equation \ref{eq: svd_S}), we further conclude that the above eigenvalues equal eigenvalues of $\mtx{G}$ and therefore $\mtx{M}$. 
     \item[2.] Let $\tvct{q}_i$ be the $i$-th column of $\bar{\mtx{Y}}_\orig^\top\tilde{\mtx{Q}}'$. By Lemma \ref{lemma: SVD_perturb} (substitute $\mtx{G}$ with $\tmtx{G}$), the i-th ($i\leq K$) eigenvector of $\tmtx{G}$ is given by $\cat{\frac{1}{\sqrt{1+ \tilde{r}_i^2 }} \tvct{p}_i' } { \frac{\tilde{r}_i}{1+\tilde{r}_i^2} \tvct{q}_i }=\cat{\frac{1}{\sqrt{1+ \tilde{r}_i^2 }} \vct{p}_i } { \frac{\tilde{r}_i}{1+\tilde{r}_i^2} \tvct{q}_i }$, where $\tilde{r}_i =\frac{\sigma_\xi}{\sqrt{m} \lambda_i'} =\frac{\sigma_\xi}{\lambda_i}$. The corresponding eigenvector of $\tmtx{M}$ is $\mtx{V}\mtx{T} \cat{\frac{1}{\sqrt{1+ \tilde{r}_i^2 }} \vct{p}_i } { \frac{\tilde{r}_i}{1+\tilde{r}_i^2} \tvct{q}_i } $. Observe that $\mtx{T}\bar{\mtx{Y}}_\orig^\top = \bar{\mtx{Y}}^\top$, therefore  $\mtx{V}\mtx{T} \cat{\frac{1}{\sqrt{1+ \tilde{r}_i^2 }} \vct{p}_i } { \frac{\tilde{r}_i}{1+\tilde{r}_i^2} \tvct{q}_i } =\mtx{V}\cat{\frac{1}{\sqrt{1+ r_i^2 }} \vct{p}_i } { \frac{r_i}{1+r_i^2} \vct{q}_i } $ which is the $i$-th eigenvector of $\mtx{M}$.
 \end{itemize}
Combining the above two leads to the conclusion that the first $K$ eigenvectors/eigenvalues of $\tmtx{M}$ and $\mtx{M}$ match. Additionally, we observe that $\colsp (\tmtx{M})\subseteq \colsp (\mtx{M})$. Therefore the span of the last $n-K$ eigenvectors of $\tmtx{M}$ is a subspace of the span of the last $mn-K$ eigenvectors of $\mtx{M}$. Since Lemma \ref{lemma: SVD_perturb} tells us that the remaining $mn-K$ eigenvalues of $\mtx{M}$ are equal, $\mtx{M}$ is identity on the span of the last $mn-K$ eigenvectors. Thus $\mtx{M}$ is identity on the span of the last $n-K$ eigenvectors of $\tmtx{M}$. Now we can conclude that $\mtx{M}^\dag\tilde{\mtx{M}} =  \tilde{\mtx{M}}^\dag\tilde{\mtx{M}}$.
\end{proof}

\begin{lemma}\label{lemma: decomp_m+}
   Suppose that the first $\frac{mn}{2}$ examples have class label $+1$ and the others have class label $-1$. Let $\mtx{L}^+\mtx{A}^+\mtx{L}^{+\top}$ (where $\mtx{L}^+\in\mathbbm{R}^{d\times 2}$)
be the eigendecomposition of $\mtx{M}^+$, then
 \begin{align}
    \vct{l}_1^+ = \mtx{V} 
    \begin{bmatrix}
        \frac{1}{\sqrt{1+\mu^2+\frac{\sigma_\xi^2}{mn}}} \\
        0\\
        \frac{\mu}{\sqrt{1+\mu^2+\frac{\sigma_\xi^2}{mn}}} \\
        \mtx{0}_{(K-2)\times 1}\\
        \frac{\sigma_\xi}{mn\sqrt{1+\mu^2+\frac{\sigma_\xi^2}{mn}}} \mtx{1}_{\frac{mn}{2}\times 1} \\
        \frac{\sigma_\xi}{mn\sqrt{1+\mu^2+\frac{\sigma_\xi^2}{mn}}} \mtx{1}_{\frac{mn}{2}\times 1} \\
    \end{bmatrix}
    ,~~~~
     \vct{l}_2^+ = \mtx{V} 
    \begin{bmatrix}
       0 \\
        \frac{\phi_1}{\sqrt{\phi_1^2+\frac{\sigma_\xi^2}{mn}}}\\
        0\\
        \mtx{0}_{(K-2)\times 1}\\
        \frac{\sigma_\xi}{mn\sqrt{\phi_1^2+\frac{\sigma_\xi^2}{mn}}} \mtx{1}_{\frac{mn}{2}\times 1} \\
        \frac{-\sigma_\xi}{mn\sqrt{\phi_1^2+\frac{\sigma_\xi^2}{mn}}} \mtx{1}_{\frac{mn}{2}\times 1} \\
    \end{bmatrix}
    ,~~~~
    a_1 = 1+\mu^2 +\frac{\sigma_\xi^2}{mn}
    ,~~~~
    a_2 =\phi_1^2 + \frac{\sigma_\xi^2}{mn}
 \end{align}
\end{lemma}

\section{Class Collapse in Supervised CL}\label{apdx: cc}
\subsection{Proof of Theorem \ref{thm: minimizer_without_cc}}
 Let $\vct{l}_\perp$ be the projection of $\vct{v}_2$ onto $\ker \mtx{M}$. By Corollary \ref{corollary: v2_proj_ker}, $\|\vct{l}_\perp\|=\Theta(\frac{\sigma_\xi}{\sqrt{mn}})$. Let $\vct{a}=\frac{mn}{\sigma_\xi^2}\vct{l}_\perp$. We can construct a $\mtx{W}^* $ that satisfies the following
\begin{align}
    \nonumber
    \mtx{W}^{*\top} \mtx{W}^* = \mtx{M}^\dag \mtx{M}^+\mtx{M}^\dag + \vct{a} \vct{a}^\top,
\end{align}
which, by Lemma \ref{lemma: minimizers}, satisfies the condition for being a minimizer of the loss. 
In the meantime,  $\mtx{W}^* $ also satisfies $\| \mtx{W}^* \vct{v}_2 \|=\Theta(1)$ by Corollary \ref{corollary: v2_proj_ker}. 
Note that both $\vct{v}_2$ and the projection of $\vct{v}_2$ onto $\colsp (\mtx{M})$ is orthogonal to $\vct{v}_k ~ (1\leq k \leq K, k\neq 2)$ as well as $\vct{v}_k ~(k>mn)$ by Lemma \ref{lemma: SVD_perturb}, therefore 
\begin{align}
\label{eq: lperp_orthogonal}
    \vct{l}_\perp \text{is also orthogonal to }~ \vct{v}_k,~ \text{for any $k$}~s.t.~ 1\leq k \leq K, k\neq 2  ~\text{and}~ k>mn.
\end{align}
Then, for $\vct{x}$ from $\mathcal{D}_\orig$ the following holds true
\begin{align}
    \nonumber
    \mtx{W}^*\vct{x} = c_0 \vct{v}_0 + c_1 y\vct{v}_1 + y_\sub c_2 \vct{v}_2 + \vct{h}_{\vct{x}} +\mtx{W}^* \vct{\xi},
\end{align}
where $c_1, c_2$ are $\Theta(1)$, and $\vct{h}_{\vct{x}}$ is orthogonal to $\vct{v}_k, k=0,\dots, K$ and $\vct{h}_{\vct{x}}\in \colsp(\mtx{M})$ (by Lemmas \ref{lemma: SVD_perturb}, \ref{corollary: v1_learned_supcon}, \ref{corollary: irrelevant}, equation \ref{eq: lperp_orthogonal} and that $\|\mtx{W}^*\vct{v}_2\|=\Theta(1)$). Let $\vct{\beta} = c_2\vct{v}_2$, then
\begin{align}
    \nonumber
    \vct{\beta}^\top \mtx{W}^* \vct{x} = y_\sub c_2^2 + \vct{\beta}^\top\mtx{W}^*\vct{\xi}.
\end{align}
With probability $\geq 1-\frac{mn}{d}$, $\vct{\xi}\notin \{\vct{v}_k\}_{k=1}^{mn}$, which indicates that $\mtx{W}^*\vct{\xi}=0$ by Lemma \ref{lemma: SVD_perturb} and equation \ref{eq: lperp_orthogonal}. Therefore we can conclude 
\begin{align}
    \nonumber
    \Pr(y_\sub\vct{\beta}^\top \mtx{W}^* \vct{x} >0|y) \geq 1-\frac{mn}{d} = 1-o(1).  
\end{align}

\subsection{Proof of Theorems \ref{thm: minimizer_with_cc} and \ref{thm: min_norm_asymp}}
 Theorems \ref{thm: minimizer_with_cc} and \ref{thm: min_norm_asymp} and  can be proved by invoking Lemma \ref{lemma: minimizers} and Corollary \ref{corollary: align_subclass}.

\section{Feature Suppression in Unsupervised CL}\label{apdx: FS}
\subsection{Feature Suppression 1}
By Lemmas \ref{lemma: minimizers} and \ref{lemma: Mtilde_Mdag}, when $p<K$, any global minimizer of $\mathcal{L}_\unsup$ satisfies
\begin{align}
\label{eq: orthonormal_basis}
    \mtx{W}^\top \mtx{W} \mtx{M} = \sum_{i=1}^p \vct{r}_i\vct{r}_i^\top,
\end{align}
where $\{\vct{r}_i\}_{i=1}^p$ can be an orthonormal basis of any $p$-dimensional subspace of $\colsp(\tmtx{M})$. By equation \ref{eq: svd_S} and Lemmas \ref{lemma: SVD_perturb} and \ref{lemma: Mtilde_Mdag}, $\mtx{M}$ and $\tmtx{M}$ each have an eigenvector $\vct{c}_1$ with eigenvalue $\frac{\sigma_\xi^2}{mn}+\phi_1^2$ and a
$\frac{1}{\sqrt{1+\frac{\sigma_\xi^2}{mn\phi_1^2}}}$ alignment with $\vct{v}_1$, with the other eigenvectors 
having no alignment with $\vct{v}_1$. Thus if we include $\vct{c}_1$ in $\{\vct{r}_i\}_{i=1}^p$ and let $\mtx{W}^\top \mtx{W}$ be null on $\ker \mtx{M}$, then the constructed $\mtx{W}$ is a minimizer of $\mathcal{L}_\unsup$ with $\Theta(1)$ alignment with $\vct{v}_1$. Now let's look at the minimum norm minimizer, which should satisfy 
\begin{align}
\nonumber
    \mtx{W}^\top \mtx{W} = \sum_{i=1}^p \vct{r}_i\vct{r}_i^\top  \mtx{M}^\dag,
\end{align}
where $\{\vct{r}_i\}_{i=1}^p$ is selected such that $\mtx{W}$ has the smallest norm. By Lemma \ref{lemma: Mtilde_Mdag}, $\{\vct{r}_i\}_{i=1}^p$ should be the p-eigenvectors of $\mtx{M}$ with largest eigenvalues (so that the inverse of the eigenvalues are among the smallest). If among  $\frac{(1+\mu^2+\phi_2^2)+\sqrt{(1+\mu^2+\phi_2)^2-4\phi_2^2} }{2}, \frac{(1+\mu^2+\phi_2^2)-\sqrt{(1+\mu^2+\phi_2)^2-4\phi_2^2} }{2}, \frac{\phi_3}{\sqrt{K-2}}, \dots, \frac{\phi_K}{\sqrt{K-2}}$ there are $p$ elements larger than $\phi_1$, then $\frac{\sigma_\xi^2}{mn}+\phi_1^2$ is not among the $p$ largest eigenvalues of $\mtx{M}$. Thus $\vct{c}_1$ is not included in $\{\vct{r}_i\}_{i=1}^p$ and the corresponding $\mtx{W}$ is orthogonal to $\vct{v}_1$.




\subsection{Feature Suppression 2}
We first present our result under slightly technical conditions.
\begin{lemma} \label{lemma:feature_suppression2}
Let $\vct{v_1}, \dots, \vct{v_C} \in \mathcal{R}^d$ be nonzero and orthogonal, $U, A$ are subspaces that are orthogonal to each other and all the $\vct{v_i}$. Suppose we have a data distribution $\mathcal{D} = \{(\vct{v_{y_i}} + \vct{u_{y_i}} + \vct{a_i}, y_i)\}_{i=1}^n \subset \mathbb{R}^d \times \{1, \dots, C\}$, where $\vct{u}_{i} \in U, \vct{a}_{i} \in A$ for all $i \in \{1, \dots, n\}$ (namely all examples in the same class $c$ share the same $\vct{v}_c$ and $\vct{u}_c$).

Denote $\vct{z}_{y_i} = \vct{v}_{y_i} + \vct{u}_{y_i}$, and let $\mtx{M}, \mtx{M}^{+}$ be the matrices defined for this dataset, and let $\mtx{Z}, \mtx{Z}^{+}$ and $\mtx{A}, \mtx{A}^{+}$ be the corresponding matrices when the data is $\{(\vct{z}_{y_i}, y_i)\}$ and $\{(\vct{a}_i, y_i)\}$, respectively. Suppose that $(\mtx{A} - \mtx{A}^{+}) \vct{v} \neq \vct{0}$ for all $\vct{v} \in \mathbb{R}^d$ s.t. $\mtx{A} \vct{v} \neq 0$ and the output dimension $p \geq C$. Then $\mtx{W}^{\top} \mtx{W} = \mtx{Z}^{\dagger}$ is the minimum norm solution to the contrastive learning objective on $\mathcal{D}$.
\end{lemma}

\begin{proof}
In this proof, we will use $\E$ to represent the empirical expectation over the dataset $\mathcal{D}$. Also, let $n_c$ denote the number of examples in class $c$.

We first derive the following expression for $\mtx{A}^{+}$:
\begin{align}
    \mtx{A}^{+} &= \E[\vct{a_i}\vct{z}_{y_i}^{\top}] \mtx{Z}^{\dagger} \E[\vct{z}_{y_i}\vct{a_i}^{\top}]
\end{align}
Define $\mtx{B} = [\sqrt{n_1}\vct{z_1} \dots, \sqrt{n_C}\vct{z_C}] \in \mathbb{R}^{d \times C}, \mtx{C} = [\vct{a}_1^{\ast}, \dots, \vct{a}_{C}^{\ast}] \in \mathbb{R}^{d \times C}$, where $n_c$ is the number of examples in class c and $\vct{a}_c^{\ast} = \frac{1}{\sqrt{n_c}} \sum_{y_i = c}\vct{a_i}$. Then
    \begin{align}
        \E[\vct{a_i}\vct{z}_{y_i}^{\top}] \mtx{Z}^{\dagger} \E[\vct{z}_{y_i}\vct{a_i}^{\top}] &= \frac{1}{n} \mtx{C}\mtx{B}^{\top} \left(\frac{1}{n} \mtx{B} \mtx{B}^{\top} \right)^{\dagger} \frac{1}{n} \mtx{B} \mtx{C}^{\top}
    \end{align}
    Now $\mtx{B}$ has full column rank, so $\mtx{B}^{\top} (\mtx{B} \mtx{B}^{\top})^{\dagger} \mtx{B} = \mtx{I}$. Thus
        \begin{align}
        \E[\vct{a_i}\vct{z}_{y_i}^{\top}] \mtx{Z}^{\dagger} \E[\vct{z}_{y_i}\vct{a_i}^{\top}] &= \frac{1}{n} \mtx{C} \mtx{C}^{\top} \\
        &= \sum_{c=1}^C \frac{n_c}{n} \E_{y_i=c}[\vct{a_i}] \E_{y_i=c}[\vct{a_i}]^{\top} \\
        &= \mtx{A}^{+}
    \end{align}

Now we show that $\mtx{W}^{\top} \mtx{W} = \mtx{Z}^{\dagger}$ is a global minimizer. It suffices to show that $\mtx{M} \mtx{W}^{\top}\mtx{W} \mtx{M} = \mtx{M}^{+}.$ Note that by assumption, we have $\inner{\vct{z}_i}{\vct{a}_j} = 0$ for all $i \in \{1, \dots, C\}, j \in \{1, \dots, n\}$, so we have
\begin{align}
    \mtx{M} \mtx{Z}^{\dagger} \mtx{M} &= \E[(\vct{z}_{y_i} + \vct{a_i}) (\vct{z}_{y_i} + \vct{a_i})^{\top}] \mtx{M}_{\ast}^{\dagger} \E[(\vct{z}_{y_i} + \vct{a_i}) (\vct{z}_{y_i} + \vct{a_i})^{\top}] \\
    &= (\mtx{Z} + \E[\vct{z}_{y_i}\vct{a_i}^{\top}] + \E[\vct{a_i}\vct{z}_{y_i}^{\top}] + \mtx{A})  \mtx{Z}^{\dagger} (\mtx{Z} + \E[\vct{z}_{y_i}\vct{a_i}^{\top}] + \E[\vct{w_i}\vct{z}_{y_i}^{\top}] + \mtx{A}) \\
    &= \mtx{Z}\mtx{Z}^{\dagger}\mtx{Z} + \mtx{Z}\mtx{Z}^{\dagger} \E[\vct{z}_{y_i}\vct{a_i}^{\top}] + \E[\vct{a_i}\vct{z}_{y_i}^{\top}] \mtx{Z}^{\dagger}\mtx{Z} + \E[\vct{a_i}\vct{z}_{y_i}^{\top}] \mtx{Z}^{\dagger} \E[\vct{z}_{y_i}\vct{a_i}^{\top}] \\
    &= \mtx{Z} + \E[\vct{z}_{y_i}\vct{a_i}^{\top}] + \E[\vct{a_i}\vct{z}_{y_i}^{\top}] + \mtx{A}^{+}\\
    &= \mtx{Z}^{+} + \E[\vct{z}_{y_i}\vct{a_i}^{\top}] + \E[\vct{a_i}\vct{z}_{y_i}^{\top}] + \mtx{A}^{+} \\
    &= \mtx{M}^{+}
\end{align}

We now want to show that this is the minimum norm solution. It is sufficient to show that $\im(\mtx{W}^{\top}\mtx{W}) = \im(\mtx{Z}^{\dagger}) = \im(\mtx{Z}) \subset \im(\mtx{M})$. Note that $\im(\mtx{M}) \subset \im(\mtx{A}) \oplus \im(\mtx{Z})$, so we can restrict $\mtx{M}$ to this subspace. We will show that $\mtx{M}$ is invertible on $\im(\mtx{A}) \oplus \im(\mtx{Z})$.
Suppose $\vct{v} = \vct{z} + \vct{a}$ with $\vct{z} \in \im(\mtx{Z}), \vct{a} \in \im(\mtx{A}), \mtx{M}\vct{v} = 0$. This implies that
\begin{align}
    \mtx{Z} \vct{z} + \E[\vct{z}_{y_i} \vct{a}_i^{\top}] \vct{a} = \vct{0} \\
    \E[\vct{a}_i \vct{z}_{y_i}^{\top}] \vct{z} + \mtx{A} \vct{a} = \vct{0}
\end{align}
Left-multiplying the first equation by $\E[\frac{n}{n_{y_i} \|\vct{v}_{y_i}\|^2} \vct{a_i} \vct{v_i}^{\top}]$, by orthogonality we have
\begin{align*}
    \vct{0} &= \E\left[\frac{n}{n_{y_i} \|\vct{v}_{y_i}\|^2} \vct{a_i} \vct{v}_{y_i}^{\top}\right] \left( \E[\vct{z}_{y_i} \vct{z}_{y_i}^{\top}] \vct{z} + \E[\vct{z}_{y_i} \vct{a}_i^{\top}] \vct{a} \right) \\
    &= \E\left[\frac{n}{n_{y_i} \|\vct{v}_{y_i}\|^2} \vct{a_i} \vct{v}_{y_i}^{\top}\right] \left( \E[(\vct{v}_{y_i} + \vct{u}_{y_i}) \vct{z}_{y_i}^{\top}] \vct{z} + \E[(\vct{v}_{y_i} + \vct{u}_{y_i}) \vct{a}_i^{\top}] \vct{a} \right) \\
    &= \E\left[\frac{n}{n_{y_i}\|\vct{v}_{y_i}\|^2} \vct{a_i} \vct{v}_{y_i}^{\top}\right] \left(\E[\vct{u}_{y_i}(  \vct{z}_{y_i}^{\top} \vct{z} + \vct{a}_i^{\top} \vct{a})] + \E[\vct{v}_{y_i}(  \vct{z}_{y_i}^{\top} \vct{z} + \vct{a}_i^{\top} \vct{a})]\right) \\
    &= \E\left[\frac{n}{n_{y_i}\|\vct{v}_{y_i}\|^2} \vct{a_i} \vct{v}_{y_i}^{\top}\right] \E[\vct{v}_{y_i}(  \vct{z}_{y_i}^{\top} \vct{z} + \vct{a}_i^{\top} \vct{a})] \\
    &= \sum_{c=1}^C \frac{1}{n n_c \|\vct{v}_c\|^2} \left(\sum_{y_i=c} \vct{a_i}\right)\vct{v}_c^{\top} \vct{v_c} \left(n_c \vct{z_c}^{\top} \vct{z} + \sum_{y_i=c} \vct{a_i}^{\top} \vct{a} \right) \\
    &= \sum_{c=1}^C \frac{1}{n n_c}\left(\sum_{y_i=c} \vct{a_i}\right) \left(n_c \vct{z_c}^{\top} \vct{z} + \sum_{y_i=c} \vct{a_i} \vct{a} \right) \\
    &= \frac{1}{n^2} \sum_{c=1}^C \frac{1}{n} \left(\sum_{y_i=c} \vct{a_i}\right)\vct{z_c}^{\top} \vct{z} + \frac{1}{n n_c}\left(\sum_{y_i=c} \vct{a_i}\right)\left(\sum_{y_i=c} \vct{a_i}\right)^{\top} \vct{a} \\
    &= E[\vct{a}_i \vct{z}_{y_i}^{\top}] \vct{z} + \mtx{A}^{+} \vct{a} 
\end{align*}
Now substituting into the second equation, we find that
\begin{align}
    (\mtx{A} - \mtx{A}^{+})\vct{a} = \vct{0} 
\end{align}
But our assumptions imply that $\vct{a} = 0$. Returning to the first equation, we now have $\mtx{Z} \vct{z} = 0$. But since $\mtx{Z}$ is diagonalizable, $\mtx{Z}$ must be invertible on its image, hence $\vct{z} = 0$. We conclude that $\vct{v} = 0$. This completes the proof.

\end{proof}

We now want to show that we can simplify some of the conditions of the previous lemma to linear independence.

\begin{lemma} \label{lemma:feature_suppression_sufficient_condition}
Suppose $d \geq  3n-2$ and $\vct{x_1}, \dots, \vct{x_n} \in \mathbb{R}^d$ are linearly independent. Then there exists a set of nonzero orthogonal vectors $\vct{v_1}, \dots, \vct{v_n}$ s.t. $\vct{x_i} = \vct{v_i} + \vct{u_i}$ and $\vct{v_i}, \vct{u_j}$ are orthogonal for all $i, \in \{1, \dots, n\}$.
\end{lemma}
\begin{proof}
WLOG assume the $\vct{x_i}$ are contained in the span of the first $n$ basis vectors. The lemma amounts to finding an orthonormal matrix $\mtx{\Omega} = \begin{pmatrix} \mtx{A} & \mtx{B} \\ \mtx{C} & \mtx{D} \end{pmatrix}$ s.t. 
\begin{align}
    \begin{pmatrix}
        \mtx{A} & \mtx{B} \\
        \mtx{C} & \mtx{D}
    \end{pmatrix}
    \begin{pmatrix}
        \mtx{X} \\
        \mtx{0} \\
    \end{pmatrix} = 
    \begin{pmatrix}
        \mtx{A} \mtx{X} \\
        \mtx{C} \mtx{X} \\
    \end{pmatrix} = 
    \begin{pmatrix}
        \mtx{\Sigma} \\
        \mtx{F} \\
    \end{pmatrix}
\end{align}
where $\mtx{\Sigma}$ is diagonal. Since the $\vct{x_i}$ are linearly independent, $\mtx{X}$ is invertible, so there exists $\mtx{A}'$ s.t. $\mtx{A}'\mtx{X}$ is diagonal.

We now want to construct a matrix $\mtx{C}$ such that $\begin{pmatrix} \mtx{A}' \\ \mtx{C}' \end{pmatrix}$ has orthogonal columns, all with norm $l > 0$. Note that $\mtx{C}'$ has at least $2n - 2$ rows. Set $\mtx{C}_{11}' = 1$, and the remaining entries in the first row so that when considering $\mtx{A}$ and the first row of $\mtx{C}'$, the first column is orthogonal to every other column. Now leave $\mtx{C}_{21}' = 0$, set $\mtx{C}_{22}' = 1$, and fill out the remaining entries in the second row so that when considering $\mtx{A}$ and the first two rows of $\mtx{C}'$, the second column is orthogonal to the remaining columns. Note that the first column remains orthogonal to all other columns. Continuing in this fashion, we can use the first $n-1$ rows of $\mtx{C}'$ to guarantee that all $n$ columns are orthogonal. Finally, suppose without loss of generality that whhen considering the $\mtx{A}'$ and the first $n-1$ rows of $\mtx{C}'$, the first column has the largest norm $l$. For each of the remaining $n-1$ rows, set the jth row to have all zero entries except possibly in the $(j+1)$-th column, which is set so that the jth column will also have norm l. Note that the columns remain orthogonal under this construction.

Now $\frac{1}{l} \begin{pmatrix} \mtx{A}' \\ \mtx{C}' \end{pmatrix}$ has orthonormal columns and $\frac{1}{l}\mtx{A}'\mtx{X}$ is still diagonal. By Gram-Schmidt, we can fill out the remaining columns of $\mtx{\Omega}$ to construct an orthonormal matrix.
\end{proof}

We now present the feature result with simplified assumptions.
\begin{lemma}
Let $Z, A$ be orthogonal subspaces. Suppose we have a data distribution $\mathcal{D} = \{(\vct{z_{y_i}} + \vct{a_i}, y_i)\}_{i=1}^n \subset \mathbb{R}^d \times \{1, \dots, C\}$, where $\vct{z}_{i} \in Z, \vct{a}_{i} \in A$ for all $i \in \{1, \dots, n\}$, and the $\vct{z}_i$ are linearly independent.

Let $\mtx{M}, \mtx{M}^{+}$ be the matrices defined for this dataset, and let $\mtx{Z}, \mtx{Z}^{+}$ and $\mtx{A}, \mtx{A}^{+}$ be the corresponding matrices when the data is $\{(\vct{z}_{y_i}, y_i)\}$ and $\{(\vct{a}_i, y_i)\}$, respectively. Suppose that $(\mtx{A} - \mtx{A}^{+}) \vct{v} \neq \vct{0}$ for all $\vct{v} \in \mathbb{R}^d$ s.t. $\mtx{A} \vct{v} \neq 0$ and the output dimension $p \geq C$. Then $\mtx{W}^{\top} \mtx{W} = \mtx{Z}^{\dagger}$ is the minimum norm solution to the contrastive learning objective on $\mathcal{D}$.
\end{lemma}
\begin{proof}
Assume that $d \geq 3C-2$, otherwise embed the distribution in a space of sufficiently large dimension. By Lemma \ref{lemma: minimizers}, the minimum norm minimizer is unaffected by adding extra dimensions. Then Lemma \ref{lemma:feature_suppression_sufficient_condition} applies, so linear independence of the $\vct{z}_{y_i}$ is sufficient to be able to construct $\vct{v}_1, \dots, \vct{v}_C, \vct{y}_1, \dots, \vct{y}_C$ satisfying Lemma \ref{lemma:feature_suppression2}, from which the conclusion follows.

\end{proof}

\section{Minimizer of The Joint Loss}\label{apdx: joint}
For simplicity we assume $\mu=0$. Same strategy can be applied to prove the theorem when $\mu\neq 0$ but a more detailed discussion on the selection of $\beta$ may be required.

By Lemmas \ref{lemma: decomp_m+} and \ref{lemma: SVD_perturb} and the expression of $\mtx{S}$ (equation \ref{eq: svd_S}), we observe that the two eigenvectors of  $\mtx{M}^+$ match two of the eigenvectors of $\mtx{M}$. By combining this with Lemma \ref{lemma: Mtilde_Mdag}, we obtain that $\beta\mtx{M}^\dag \mtx{M}^+ +(1-\beta)\mtx{M}^\dag \tmtx{M} = \vct{l}_1^+ \vct{l}_1^{+\top} + \vct{l}_2^+ \vct{l}_2^{+\top}$ on $\linspan(\{\vct{l}_1^+, \vct{l}_2^+\})$ and $\beta\mtx{M}^\dag \mtx{M}^+ +(1-\beta)\mtx{M}^\dag \tmtx{M} =  (1-\beta)\tmtx{M}^\dag \tmtx{M}$ on $\linspan(\{\vct{l}_1^+, \vct{l}_2^+\})^\perp$. Thus the eigenvalues of $\beta\mtx{M}^\dag \mtx{M}^+ +(1-\beta)\mtx{M}^\dag \tmtx{M}$ are $1,1, 1-\beta, 1-\beta, \dots, 1-\beta$. When $\beta\in(0,1)$, $\vct{l}_1^+ $ and $ \vct{l}_2^+$ are the two eigenvectors of $\beta\mtx{M}^\dag \mtx{M}^+ +(1-\beta)\mtx{M}^\dag \tmtx{M}$ with largest eigenvalues. For the remaining eigenvectors, since they have equally large eigenvalues (same as analyzed in \ref{apdx: FS}), the minimum norm minimizer will select the largest $p-2$ of them. In the setting of Theorem \ref{thm: joint} $(1-\beta)(\phi_2^2+\frac{\sigma_\xi^2}{mn})$ is one of the $p-2$ largest of the remaining. As a result, both components aligned with $\vct{v}_1$ and $\vct{v}_2$ are selected by the minimum norm minimizer of the joint loss.


\section{Early in Training Subclasses Are Learned}\label{apdx: early}
We assume $\sigma_\xi = O(1)$.

\subsection{Lemmas}

\begin{lemma}[Laurent-Massart \cite{laurent2000adaptive} Lemma 1, page 1325]\label{lemma:LM}
Let $v_1, \dots, v_d$ be i.i.d. Gaussian variables drawn from $\mathcal{N}(0,1)$. Let $\vct{a}=(a_1, \dots, a_d)$ be a vector with non-negative components. Let $Z=\sum_{i=1}^d a_i (v_i^2-1)$. The following inequalities hold for any positive $t$:
\begin{align}
    \nonumber
    &\Pr(Z\geq 2\|\vct{a}\|_2\sqrt{t}+2\|\vct{a}\|_{\infty}t)\leq e^{-t},\\
    &\Pr(Z\leq -2\|\vct{a}\|_2\sqrt{t})\leq e^{-t}.
\end{align}
\end{lemma}

\begin{lemma}[Mills' ratio. Exercise 6.1 in \cite{shorack2000probability}.]\label{lemma:mills}
Let $v$ be a Gaussian random variable drawn from $\mathcal{N}(0,1)$. Then for all $\lambda>0$,
\begin{align}
    \nonumber
    \frac{\lambda}{\lambda^2+1}\frac{1}{\sqrt{2\pi}} e^{-\frac{\lambda^2}{2}}<\Pr(v\geq \lambda) < \frac{1}{\lambda}\frac{1}{\sqrt{2\pi}} e^{-\frac{\lambda^2}{2}}.
\end{align}
\end{lemma}

\begin{corollary}\label{corollary: q_z}
Given a vector $\vct{q}$, and a random vector $\vct{z}$ drawn from $\mathcal{N}(0,\frac{\sigma}{d}\mathbf{I}_d)$, w.p. $\geq 1-O(\frac{\delta}{\sqrt{\log 1/\delta} })$, $ |\vct{z}^{\top}\vct{q} |=O(\frac{\|\vct{q}\|\sigma\sqrt{\log \frac{1}{\delta}} }{\sqrt{d}})$.
\end{corollary}
\begin{proof}
This can be proven by considering the fact that $\vct{q}^\top \vct{z}$ is a Gaussian variable and applying Lemma \ref{lemma:mills}.
\end{proof}

\begin{lemma}\label{lemma:random_W_align_u}
Let each element of $\mtx{W}_0\in\mathbbm{R}^{p\times d}$ be randomly drawn from $\mathcal{N}(0,\frac{\sigma_0^2}{d}\mathbf{I}_d)$. Let $\vct{u}\in\mathbbm{R}^d$ be a unit vector. With probability at least $1-\delta$, we have
\begin{align}
\nonumber
    \|\mtx{W}_0 \vct{u}\| \geq & \sigma_0\sqrt{\frac{p}{d}} \sqrt{1-2\sqrt{\frac{\ln 2/\delta}{p}}}\\
\nonumber
\|\mtx{W}_0 \vct{u}\| \leq & \sigma_0\sqrt{\frac{p}{d}} \sqrt{1+2\sqrt{\frac{\ln 2/\delta}{p}}+2\frac{\ln 2/\delta}{p}}.
\end{align}
\end{lemma}
\begin{proof}
Firstly rewrite $\|\mtx{W}_0 \vct{u}\|$ as
\begin{align}
    \nonumber
    \|\mtx{W}_0 \vct{u}\| = \sqrt{\sum_{i=1}^p (\vct{w}_0^{(i)\top}\vct{u} )^2 } = \sigma_0\sqrt{\frac{p}{d}} \sqrt{ \frac{1}{p}\sum_{i=1}^p (\frac{\sqrt{d}}{\sigma_0}\vct{w}_0^{(i)\top}\vct{u} )^2 }.
\end{align}
By spherical symmetric, each $\frac{\sqrt{d}}{\sigma_0}\vct{w}_0^{(i)\top}\vct{u}$ is a random Gaussian variable drawn from $\mathcal{N}(0,1)$. By lemma \ref{lemma:LM} we have 
\begin{align}
\nonumber
    \Pr \left(\frac{1}{p}\sum_{i=1}^p (\frac{\sqrt{d}}{\sigma_0}\vct{w}_0^{(i)\top}\vct{u} )^2 \leq 1-2\sqrt{\frac{\ln 2/\delta}{p}} \right)\leq & \delta/2\\
    \nonumber
    \Pr \left(\frac{1}{p}\sum_{i=1}^p (\frac{\sqrt{d}}{\sigma_0}\vct{w}_0^{(i)\top}\vct{u} )^2 \geq 1+2\sqrt{\frac{\ln 2/\delta}{p}}+2\frac{\ln 2/\delta}{p} \right)\leq & \delta/2
\end{align}
 which completes the proof.
\end{proof}

\subsection{Proof of Theorem \ref{thm: early_in_training}}

We assume the dataset satisfies the condition in Section \ref{apdx: prop_minimizers} (wich holds with probability $1-O(\frac{m^2n^2}{d})$). 
Let $\mtx{L}\mtx{A}\mtx{L}^\top$ (where $\mtx{C}\in\mathbbm{R}^{d\times mn}$)
be the eigendecomposition of $\mtx{M}$.
By equation \ref{eq: svd_S} and Lemma \ref{lemma: decomp_m+} and Lemma \ref{lemma: SVD_perturb}, we observe that when $\mu\neq 0$ all but three of $\mtx{M}$'s eigenvectors are orthogonal to $\vct{l}_1^+$, $\vct{l}_2^+$. W.L.O.G., let $\vct{l}_1, \vct{l}_2$ and $\vct{l}_3$ be those three eigenvectors. The corresponding three eigenvalues are all constants. 
Let $\vct{l}_3^+$ be a unit vector in $\linspan(\{\vct{l}_1, \vct{l}_2, \vct{l}_3\})-\linspan(\{\vct{l}_1^+, \vct{l}_2^+\}) $. Decompose $\vct{v}_2$ as $\frac{\mu}{\sqrt{1+\mu^2+\frac{\sigma_\xi^2}{mn}}}\vct{l}_1^+ + \frac{\sqrt{1+\frac{\sigma_\xi^2}{mn} }}{\sqrt{1+\mu^2+\frac{\sigma_\xi^2}{mn}}}\vct{l}_\perp$ where $\vct{l}_\perp $ is a unit vector that is orthogonal to $\vct{l}_1^+$. Since $\vct{v}_2 \perp \vct{l}_2^+$, we have $\vct{l}_\perp \perp  \vct{l}_2^+$ thus $\mtx{M}^+\vct{l}_\perp=0$.

 Define
\begin{align}
\nonumber
\sqrt{\mtx{M}} \coloneqq & \mtx{L}\sqrt{\mtx{A}} \\
\nonumber
\Gamma_i(t) \coloneqq & \| \mtx{W}_t\vct{l}_i^+ \|, ~~i=1,2,3\\
\nonumber
\Gamma_\perp(t) \coloneqq & \| \mtx{W}_t\vct{l}_\perp \|\\
\nonumber
\Gamma_{:3}(t) \coloneqq &\sqrt{\sum_{i=1}^3 \| \mtx{W}_t\vct{l}_i^+ \|^2} \\
\nonumber
\mtx{B} \coloneqq & [\sqrt{a}_4 \vct{l}_4 ~~ \sqrt{a}_5 \vct{l}_5~~\dots ~~ \sqrt{a}_{mn} \vct{l}_{mn}]\\
\nonumber
\Gamma_B(t) \coloneqq & \| \mtx{W}_t \mtx{B}  \|_F\\
\nonumber
s \coloneqq & \| \sqrt{\mtx{M}} \| = O(1)\\
\nonumber
h \coloneqq & \| \sqrt{\mtx{M}}^\top \mtx{B} \|=\sqrt{\sum_{i=4}^{mn} a_i^2 } = \sqrt{\sum_{i=3}^{K}(\frac{\sigma_\xi^2}{mn}+\frac{\phi_i^2}{(K-2)})^2 + (mn-K)\frac{\sigma_\xi^4}{m^2n^2}} \\
\nonumber
= & O(\sqrt{\frac{\sigma_\xi^2}{mn} + \frac{1}{K} +\frac{\sigma_\xi^4}{mn}}) = O(1)
\end{align}
Then we bound $\|\mtx{W}_t\sqrt{\mtx{M}} \|_F$
\begin{align}
\nonumber
\|\mtx{W}_t\sqrt{\mtx{M}} \|_F= &\|\mtx{W}_t\mtx{L}\sqrt{\mtx{A}}\|_F \\
\nonumber
= &\|[\mtx{W}_t \sqrt{a}_1 \vct{l}_1~~~ \mtx{W}_t \sqrt{a}_2 \vct{l}_2 ~~\dots ~~ \mtx{W}_t \sqrt{a}_{mn} \vct{l}_{mn}   ]\|_F \\
\nonumber
= & \sqrt{\sum_{i=1}^3 \| \mtx{W}_t \sqrt{a}_i \vct{l}_i \|^2 + \sum_{i=4}^{mn} \| \mtx{W}_t \sqrt{a}_i \vct{l}_i \|^2  }\\
\nonumber
\leq & \sqrt{c\Gamma_{:3}(t)^2 + \Gamma_B(t)^2 },
\end{align}
where $c$ is a constant because $a_1,a_2,a_3$ are all $O(1)$ (by Lemma \ref{lemma: SVD_perturb}) and each $\vct{l}_i$ ($i=1,2,3$) is a linear combination of $\vct{l}_1^{++}, \vct{l}_2^{++}, \vct{l}_3^+$ with $O(1)$ coefficients, with $\vct{l}_1^{++}, \vct{l}_2^{++}$ representing the projections of $\vct{l}_1^{+}, \vct{l}_2^{+}$ onto $\linspan(\{ \vct{l}_i \}_{i=1}^3)$.

By the rule of gradient descent we have
\begin{align}
    \label{eq: rule_gd}
    \mtx{W_{t+1}} = &\mtx{W_t} + \eta(4 \mtx{W_t} \mtx{M}^{+} - 4 \mtx{W_t} \mtx{M} \mtx{W_t}^{\top} \mtx{W_t} \mtx{M})\\
    \nonumber
    = & \mtx{W_t} + 4 \eta\mtx{W_t} \mtx{M}^{+} - 4\eta \mtx{W_t} \mtx{M} \mtx{W_t}^{\top} \mtx{W_t} \mtx{M}m
\end{align}
This is followed by Lemma \ref{lemma: recurr}.
\begin{lemma} \label{lemma: recurr}
By the update rule of GD we have the following recurrence relations  
\begin{align}
    \nonumber
    \Gamma_1(t+1) \geq & (1+4\eta a_1^+) \Gamma_1(t)-4\eta ( c\Gamma_{:3}^2(t)+\Gamma_B(t)^2 )^{3/2}s \\
    \nonumber
    \Gamma_1(t+1) \leq & (1+4\eta a_1^+) \Gamma_1(t)+4\eta ( c\Gamma_{:3}^2(t)+\Gamma_B(t)^2 )^{3/2}s \\
    \nonumber
    \Gamma_2(t+1) \leq & (1+4\eta a_2^+) \Gamma_2(t)+4\eta ( c\Gamma_{:3}^2(t)+\Gamma_B(t)^2 )^{3/2}s \\
    \nonumber
    \Gamma_3(t+1) \leq & \Gamma_3(t)+4\eta ( c\Gamma_{:3}^2(t)+\Gamma_B(t)^2 )^{3/2}s \\
    \nonumber
    \Gamma_\perp(t+1) \leq & \Gamma_\perp(t)+4\eta ( c\Gamma_{:3}^2(t)+\Gamma_B(t)^2 )^{3/2}s \\
    \nonumber
    \Gamma_B(t+1) \leq & \Gamma_B(t) + 4\eta ( c\Gamma_{:3}^2(t)+\Gamma_B(t)^2 )^{3/2} h.
\end{align}    
\end{lemma}

Then we prove the following Lemma 
\begin{lemma}\label{lemma: Gamma_init}
At initialization the following holds with probability $\geq 1-O(\frac{1}{\poly(p)})$
\begin{align}
    \nonumber
    \frac{\Gamma_B(0)}{\Gamma_1(0)} =& O(1)\\
    \nonumber
    \Gamma_i(0) =& \sigma_0\sqrt{\frac{p}{d}}\left( 1\pm O(\sqrt{\frac{\log p}{p}}) \right), ~~ i=1,2,3 \\
    \nonumber
    \Gamma_\perp(0) =& \sigma_0\sqrt{\frac{p}{d}}\left( 1\pm O(\sqrt{\frac{\log p}{p}}) \right) 
\end{align}
\end{lemma}
\begin{proof}
We first bound $\Gamma_B(0)$
\begin{align}
    \nonumber
    \Gamma_B(0) = & \sqrt{\sum_{i=4}^{mn} a_i \| \mtx{W}_0 \vct{l}_i \|^2 }\\
    \nonumber
    = & \sqrt{\sum_{i=4}^{mn} a_i \sum_{j=1}^p \| \mtx{w}_{0,j}^\top \vct{l}_i \|^2 }\\
    \nonumber
    \leq & \sqrt{\frac{\phi_{\max}^2}{K-2}\sum_{i=4}^{K+1} \sum_{j=1}^p \| \mtx{w}_{0,j}^\top \vct{l}_i \|^2 +\frac{\sigma_\xi^2}{mn}\sum_{i=K+2}^{mn} \sum_{j=1}^p \| \mtx{w}_{0,j}^\top \vct{l}_i \|^2 }  \\
    \nonumber
    = & O( \sqrt{\frac{p\sigma_0^2}{d} +\sigma_\xi^2\frac{p\sigma_0^2}{d}} ) ~~~~\textcircled{1} \\
    \nonumber
    = &O( \sigma_0\sqrt{\frac{p}{d}} ).
\end{align}
Inequality \textcircled{1} holds with probability $\geq 1-O(\frac{1}{\poly(mnp)})$. It is obtained by obsreving that $\mtx{w}_{0,j}^\top \vct{l}_i$'s are independent Gaussian variables (by the orthogonality of $\vct{l}_i$'s) and applying Lemma \ref{lemma:LM} to the sum of $\|\mtx{w}_{0,j}^\top \vct{l}_i\|^2$'s.

By Lemma \ref{lemma:random_W_align_u} and the above, at initialization the following holds with probability $\geq 1-O(\frac{1}{\poly(p)}+\frac{1}{\poly(mnp)})$
\begin{align}
    \nonumber
    \frac{\Gamma_B(0)}{\Gamma_1(0)} =& O(1)\\
    \nonumber
    \Gamma_i(0) =& \sigma_0\sqrt{\frac{p}{d}}\left( 1\pm O(\sqrt{\frac{\log p}{p}}) \right), ~~ i=1,2,3 \\
    \nonumber
    \Gamma_\perp(0) =& \sigma_0\sqrt{\frac{p}{d}}\left( 1\pm O(\sqrt{\frac{\log p}{p}}) \right) .
\end{align}
\end{proof}

Let $\psi, \psi_B$ be constants. Define
\begin{align}
\nonumber
\pi \coloneqq &\frac{\Gamma_B(0)}{\Gamma_1(0)} = O(1) ~~\text{by Lemma \ref{lemma: Gamma_init}}\\
    \nonumber
    \tau \coloneqq & (c (1+2(1+\psi)^2) + (\pi +\psi_B )^2)^{3/2} = \Theta(1).
\end{align}

Let $\gamma$ be a constant satisfying the following
\begin{align}
    \nonumber
    \gamma \leq \min \left\{\sqrt{\frac{(a_1^+-a_2^+)\psi}{\tau(s+s\psi)}}, \sqrt{\frac{a_1^+ \psi}{\tau(s+s\psi)}} , \sqrt{\frac{a_1^+\psi_B}{\tau(h+s\psi_B)}}, \sqrt{\frac{a_1^+ - a_2^+}{\tau s} }\right\} .
\end{align}
Note that $a_1^+ -a_2^+ >0$ because $\mu^2+1 > \phi_1^2$. Additionally, we define the following shorthand
\begin{align}
    \nonumber
    \epsilon \coloneqq & 4\eta\tau \gamma^2 s, \\
    \nonumber
    \epsilon_B \coloneqq & 4\eta \tau \gamma^2 h\\
    \nonumber
    \alpha \coloneqq & 1+4\eta a_1^+ - \epsilon\\
    \nonumber
    \hat{\alpha} \coloneqq & 1+4\eta a_1^+ + \epsilon\\
    \nonumber
    \kappa_2 \coloneqq & \frac{1+4\eta a_2^+}{\alpha} <1 \text{ because $\mu^2+1 > \phi_1^2$ }\\
    \nonumber
    \kappa_3 \coloneqq & \frac{1}{\alpha}\\
    \nonumber
    \kappa_\perp \coloneqq & \frac{1}{\alpha}\\
    \nonumber
    \kappa_B \coloneqq & \frac{1}{\alpha}.
\end{align}

Now we are ready to prove the following Lemma.
\begin{lemma}\label{lemma: smaller_than_gamma}
    If $\forall t \leq T$, $\Gamma_1(t)\leq \gamma$. For any constants $\psi, \psi_B$, the following holds $\forall t\leq T+1$ with probability $1-O(\frac{1}{\poly(p)} )$,
    \begin{itemize}
    \item $ \Gamma_1(t)\geq \alpha^t \Gamma_1(0)$
    \item $ \Gamma_1(t)\leq \hat{\alpha}^t \Gamma_1(0)$.
\item $\Gamma_i(t) \leq (\kappa_i^{t}+\psi) \Gamma_1(t), ~~i=2,3$.
\item $\Gamma_\perp(t) \leq (\kappa_\perp^{t}+\psi) \Gamma_1(t)$.
\end{itemize}
\end{lemma}
\begin{proof}
Let $S(k)$ be the following statement: $\forall t'$ such that $0\leq t' \leq k$, the following holds
\begin{itemize}
    \item $ \Gamma_1(t')\geq \alpha^{t'} \Gamma_1(0) $,
     \item $ \Gamma_1(t)\leq \hat{\alpha}^t \Gamma_1(0)$,
\item $\Gamma_i(t') \leq (\kappa_i^{t'}+\psi )\Gamma_1(t'), ~~i=2,3$,
\item $\Gamma_\perp(t') \leq (\kappa_\perp^{t'}+\psi) \Gamma_1(t')$,
\item $\Gamma_B(t') \leq (\kappa_B^{t'}\pi +\psi_B) \Gamma_1(t') $.
\end{itemize}
By Lemma \ref{lemma: Gamma_init}, $S(0)$ holds with high probability. Next we show that, $\forall t \in [0, T+1]$, if $S(t-1)$ holds then $S(t)$ also holds. By Lemma \ref{lemma: recurr}, the induction hypothesis and $\kappa_2,\kappa_3, \kappa_\perp, \kappa_B < 1$ , $\Gamma_1(t-1)\leq \gamma$, we have the following
\begin{align}
\label{eq:Gamma_1_lower}
    \Gamma_1(t) \geq &\alpha \Gamma_1(t-1) \\
\label{eq:Gamma_1_upper}
    \Gamma_1(t) \leq &\hat{\alpha} \Gamma_1(t-1) \\
    \nonumber
    \Gamma_2(t) \leq & \left((1+4\eta a_2^+)(\kappa_2^t +\psi ) +\epsilon \right) \Gamma_1(t) \\
    \nonumber
    \Gamma_3(t) \leq & \left((\kappa_3^t +\psi ) +\epsilon \right) \Gamma_1(t) \\
    \nonumber
    \Gamma_\perp(t) \leq & \left((\kappa_\perp^t +\psi ) +\epsilon \right) \Gamma_1(t) \\
    \nonumber
    \Gamma_B(t) \leq & \left((\kappa_B^t\pi +\psi_B ) +\epsilon_B \right) \Gamma_1(t) .
\end{align}
By the construction of our $\kappa$'s, $\alpha$'s, $\epsilon$'s and $\psi$'s, the last three items in statement $S(t)$ hold. Combining the induction hypothesis with equations \ref{eq:Gamma_1_lower} and \ref{eq:Gamma_1_upper} yields the first two items in $S(t)$, which completes the proof.
\end{proof}

Now we are ready to prove the theorem.
\begin{theorem}
If $\sigma_0\sqrt{\frac{p}{d}}=o(1)$ and $\sigma_\xi=o(1)$, with probability at least $1-O(\frac{m^2n^2}{d}+\frac{1}{\poly(p)})=1-o(1)$, the following holds
\begin{itemize}
    \item $\|\mtx{W}_0 \vct{v}_2\|=o(1)$.
    \item $\exists t=O(\ln(\frac{1}{\sigma_0} \sqrt{\frac{d}{p}}))$, s.t. $\|\mtx{W}_t \vct{v}_2\|=\Omega(1)$.
\end{itemize}
\end{theorem}
\begin{proof}
$\|\mtx{W}_0 \vct{v}_3\|=o(1)$ follows Lemma \ref{lemma:random_W_align_u} and the assumption that $\sigma_0\sqrt{\frac{p}{d}}=o(1)$.  Select a constant $\psi$ such that $ \psi < \frac{\mu}{\sqrt{1+\frac{\sigma_\xi^2}{mn}}} $. Note that $\frac{\mu}{\sqrt{1+\frac{\sigma_\xi^2}{mn}}}-\psi=\Theta(1)$. Let $T = \lfloor \frac{\ln(\gamma/\Gamma_1(0) )}{\ln \alpha} \rfloor = \Theta(\ln \frac{1}{\sigma_0} \sqrt{\frac{d}{p}} )$. There are two cases to consider.
\begin{itemize}
    \item If $\forall t\leq T$, $\Gamma_1(t)\leq \gamma$, by Lemma \ref{lemma: smaller_than_gamma} we have $\Gamma_1(T+1) \geq \gamma$ and $\Gamma_\perp(T+1) \leq (o(1) +\psi )\Gamma_1(T+1) $. Then
    \begin{align}
        \nonumber
        \|\mtx{W}_{T+1}\vct{v}_2\| \geq & \frac{\mu}{\sqrt{1+\mu^2+\frac{\sigma_\xi^2}{mn}}} \| \mtx{W}_{T+1}\vct{l}_1^+ \| - \frac{\sqrt{1+\frac{\sigma_\xi^2}{mn}}}{\sqrt{1+\mu^2+\frac{\sigma_\xi^2}{mn}}}  \| \mtx{W}_{T+1}\vct{l}_\perp \| \\
        \nonumber
        \geq & (\frac{\mu}{\sqrt{1+\mu^2+\frac{\sigma_\xi^2}{mn}}} - \frac{\sqrt{1+\frac{\sigma_\xi^2}{mn}}}{\sqrt{1+\mu^2+\frac{\sigma_\xi^2}{mn}}} \psi - o(1) )\gamma\\
        \nonumber
        = & \Omega(1).
    \end{align}
    \item If $\exists t \leq T~s.t.~ \Gamma_1(t) >\gamma $, we define $T^*=\frac{\ln(\frac{\gamma}{\Gamma_1(0)})}{\ln\hat{\alpha}} $ and $t^* = \min t ~s.t.~ \Gamma_1(t) >\gamma $. It follows that  $\forall t \leq t^*-1, \Gamma_1(t) \leq \gamma$. Then we can apply Lemma \ref{lemma: smaller_than_gamma} to obtain $\Gamma_1(t^*) \leq \hat{\alpha}^{t^*}\Gamma_1(0) $. If $t< T^*$, the above yields $\Gamma_1(t^*) < \gamma$, which contradicts the definition of $t^*$. Therefore we conclude $t^*\geq T^*$. Lemma \ref{lemma: smaller_than_gamma} also tells that $ \Gamma_\perp(t^*) \leq  (\kappa_\perp^{t^*} +\psi) \Gamma_1(t^*)$. Since $t^*\geq T^*$ and $\kappa_\perp<1$, we have $\kappa_\perp^{t^*}\leq (\frac{\Gamma_1(0)}{\gamma})^{\frac{\ln(1/\kappa_3)}{\ln\hat{\alpha}}} = o(1)$. Therefore $\Gamma_\perp(t^*) \leq  (o(1) +\psi) \Gamma_1(t^*)$. By the definition of $t^*$, $\Gamma_1(t^*)>\gamma$. Then we can lower bound  $\|\vct{W}_{t^*}\vct{v}_2\|$ in the same way as in the previous case
    \begin{align}
        \nonumber
        \|\mtx{W}_{t^*}\vct{v}_2\| \geq & \frac{\mu}{\sqrt{1+\mu^2+\frac{\sigma_\xi^2}{mn}}} \| \mtx{W}_{t^*}\vct{l}_1^+ \| - \frac{\sqrt{1+\frac{\sigma_\xi^2}{mn}}}{\sqrt{1+\mu^2+\frac{\sigma_\xi^2}{mn}}}  \| \mtx{W}_{t^*}\vct{l}_\perp \| \\
        \nonumber
        \geq & (\frac{\mu}{\sqrt{1+\mu^2+\frac{\sigma_\xi^2}{mn}}} - \frac{\sqrt{1+\frac{\sigma_\xi^2}{mn}}}{\sqrt{1+\mu^2+\frac{\sigma_\xi^2}{mn}}} \psi - o(1) )\gamma\\
        \nonumber
        = & \Omega(1).
    \end{align}
\end{itemize}
\end{proof}

\newpage 
\section{Experimental Setup and Additional Experimental Results}\label{apdx:exp}

\subsection{Datasets}\label{apdx:datasets}

\textbf{CIFAR-10/100.} The two datasets each consist of 60000 32x32 colour images \cite{krizhevsky2009learning}. In the case of CIFAR-10, the `classes' refer to the original 10 classes defined in the dataset, while we define `subclasses' as two subclasses: vehicles (airplane, automobile, ship, truck) and animals (bird, cat, deer, dog, frog, horse). On CIFAR-100, we refer to the 10 super-classes (e.g. aquatic mammals, fish, flowers) as our 'classes' and the 100 classes as our 'sub-classes'. These two datasets illustrate a natural setting where class collapse is extremely harmful, as it results in learning representations that do not capture much of the semantically relevant information from the data. 

\textbf{MNIST RandBit.} The MNIST RandBit dataset \citet{chen2021intriguing} is created by setting $n$, the \# of bits that specifies how easy the useless feature will be. Larger $n$ makes the feature more discriminative, thus `easier' and more problematic for feature suppression. An extra channel is concatenated to MNIST images where each value in the feature map corresponds to a random integer between $0$ and $2^n$.

\textbf{CIFAR-10/100 RandBit.} The two datasets are constructed in a similar way as MNIST RandBit, but with images from CIFAR-10/100.

\subsection{Training details}\label{apdx:training_details}

For the experiments on CIFAR-10/100 or CIFAR-100 RandBit, we use a ResNet-18 trained with (Momentum) SGD using learning rate = $0.01$ and momentum = $0.9$. We train with batch size set to 512 for 1000 epochs. For data augmentations, we consider the standard data augmentations from \citet{cl_simclr}. 

For the feature suppression experiments on MNIST RandBit, we directly use the code provided by \citet{chen2021intriguing}. We consider a 5-Layer convolutional network. For our data augmentations, we consider the standard set of data augmentations for images and do not alter the useless feature (extra channel concatenated of RandBits).


\subsection{Details and additional experiments on varying embedding size}\label{apdx:additional_embedding_size}

In the experiments presented in Table \ref{tab:cifar_10_increasing_width}, we vary the width, denoted by $w$, of the ResNet, which is controlled by the number of convolutional layer filters. For width $w$, there are $w$, $2w$, $4w$, $8w$ filters in each layer of the four ResNet blocks. 

\begin{table}[!t]
    \caption{increasing $k$ improves both subclass and class accuracies on CIFAR-10 RandBit. }
    \label{tab:increasing_k}
\vspace{1mm}
    \centering
    \begin{tabular}{|c|c|c|}
    \hline
       $k$  & Sub Acc & Acc   \\
    \hline
        1 & 34.38 & 86.73 \\
    \hline
        16 & 58.12 & 94.09\\
    \hline
    \end{tabular}
    \vspace{-2mm}
\end{table}

In addition, we explore an alternative way of varying the embedding size, which isolates the effect of the last layer's embedding size from the size of the lower layers. Specifically, we set the width parameter $w=4$ and multiply the width of only the last ResNet block by a factor $k$. It is worth noting that doing this requires a much smaller total number of parameters. Table \ref{tab:increasing_k} presents the results on CIFAR-10 RandBit. We observe that increasing $k$ also effectively improves the accuracy. Although the improvement is not as substantial as in the previous case where we increase $w$, it confirms the same trend predicted by the theory, supporting the conclusion that increasing the embedding size alleviates feature suppression.

\section{Potential Approaches to Theoretical Characterization of Class Collapse in (S)GD}\label{apdx:potential}

The most crucial aspect that remains to be tackled is how (S)GD unlearns subclass features that have already been learned early in training. We offer two potential approaches that could help in achieving this goal. 

\subsection{Through implicit bias of (S)GD in matrix factorization}\label{apdx:matrix_fac}

The contrastive loss we are considering can be reformulated as a matrix factorization objective:
\begin{align}
\label{eq:matrix_fac}
    \min f(\mtx{W}^{\top}\mtx{W}) = \frac{1}{n^2}\sum_{i,j}(\vct{x}_i^{\top}\mtx{W}^\top \mtx{W} \vct{x}_j - a_{ij})^2,
\end{align}
where 
$$
a_{ij}\coloneqq 
\begin{cases}
2, ~~ if ~~ x_i \text{ and } x_j \text{ are from the same class} \\
0. ~~ else
\end{cases}
$$
This opens up the possibility of leveraging the rich literature on matrix factorization to find a solution. Since we have already proven in Theorems 4.4 and 4.7 that the minimum norm minimizer of the loss function exhibits class collapse, and our experiments confirm that (S)GD does converge to a minimizer that exhibits class collapse, it is reasonable to investigate whether the implicit bias of (S)GD in our setting, specifically matrix factorization, is to seek the minimum norm solution.

We note that understanding the implicit bias in matrix factorization is a longstanding pursuit in the machine learning community. \cite{gunasekar2017implicit} have provided empirical and theoretical evidence that under certain conditions, gradient descent converges to the minimum nuclear norm solution. Therefore, one can examine whether similar existing results can be applied to our setting and then combine that with our Theorems 4.4 or 4.7 to show class collapse in GD. However, \cite{arora2019implicit} and \cite{razin2020implicit} suggested that the implicit bias may be explained by rank rather than norm when the depth of a network $\geq 2$.

\subsection{Through analyzing the two terms in the gradient}\label{apdx:two_terms}

Let’s take a closer look at the update of the weights in GD, i.e., learning rate times minus gradient $4\mtx{W}\mtx{M}^+ - 4\mtx{W}\mtx{M}\mtx{W}^{\top} \mtx{W}\mtx{M}$ (see Equation \label{eq: rule_gd}). There are two terms $4\mtx{W}\mtx{M}^+$ and $- 4\mtx{W}\mtx{M}\mtx{W}^{\top} \mtx{W}\mtx{M}$ which play different roles in the high level. Here $\mtx{M}^+$ is the covariance of class centers, and $\mtx{M}$ is the covariance of all training examples, as defined in Definition \ref{def: Ms}. 

\textbf{Term 1 ($4\mtx{W}\mtx{M}^+$)}: The first term aligns the weights with $\mtx{M}^+$ which has alignment with the subclass feature. This aligning effect of term 1 has already been theoretically characterized in our proof (Appendix \ref{apdx: early}) for Theorem \ref{thm: early_in_training}.

\textbf{Term 2 ($- 4\mtx{W}\mtx{M}\mtx{W}^{\top}\mtx{W}\mtx{M}$)}: Although the effect of the second term is not entirely straightforward, we can gain some intuition by considering the simplest case where the embedding is one-dimensional. In this case, the second term takes the form of a negative scalar times $WM$, which can be seen as trying to `discourage' alignment with $\mtx{M}$, the covariance of all training examples. 

In our numerical experiment, we observe that the ratio between norms of term 1 and term 2 initially starts at a very large value, then decreases until it reaches a plateau around 1, as shown in Figure \ref{fig:norm_grad}. Interestingly, the point at which the ratio dropped to around 1 coincided almost precisely with the peak of the projection of the weights $\mtx{W}$ onto the subclass feature. This leads us to the following intuition, which may serve as a proof sketch for showing class collapse at the end of training. In the following, we first describe what happens in the early phase (which we have already proven in the paper), then outline the high-level idea of how subclasses are eventually unlearned.

\textbf{Phase I where the model learns the subclass feature}: We have already proved this part in Appendix \ref{apdx: early}. In summary, the intuition is that early in training, the scale of term 1 dominates over term 2, aligning the model with $\mtx{M}^+$, which in turn aligns with the subclass feature. Therefore, the model learns the subclass feature during this phase.

\begin{figure}[!t]
    \centering
    \includegraphics[width=.25\textwidth]{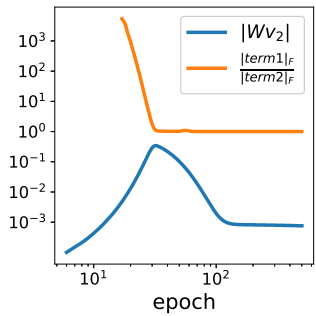}
    \caption{We plot the ratio between norms of term 1 and term 2 in orange, and the projection of the weights $\mtx{W}$ onto the subclass feature ($\|\mtx{W}\vct{v}_2\|$) in blue. The ratio between the two norms  initially starts at a very large value, then decreases until it reaches a plateau around 1. The point at which the ratio dropped to around 1 coincided almost precisely with the peak of $\|\mtx{W}\vct{v}_2\|$.}
    \label{fig:norm_grad}
\end{figure}

\textbf{Phase II where the model unlearns the subclass feature but the class feature remains}: Note that the scale of the second term also increases during Phase I as $\mtx{M}$ and $\mtx{M}^+$ share certain components. Once the scale of term 2 reaches that of term 1, the effect of term 2 becomes more pronounced and Phase II begins. Since $\mtx{M}$ exhibits a stronger correlation with the subclass feature than $\mtx{M}^+$ does, the overall effect of the sum of term 1 and term 2 is to reduce alignment with the subclass feature. Thus, over time, the model unlearns the subclass feature. In contrast, for the class feature, $\mtx{M}^+$ has a stronger correlation, causing GD to continue aligning the model with the class feature.

\end{document}